\documentclass{article}
\PassOptionsToPackage{numbers, compress}{natbib}

\usepackage[main,final]{neurips_2026}

\usepackage[utf8]{inputenc} 
\usepackage[T1]{fontenc}    
\usepackage{hyperref}       
\usepackage{url}            
\usepackage{booktabs}       
\usepackage{amsfonts}       
\usepackage{nicefrac}       
\usepackage{microtype}      
\usepackage{xcolor}         

\usepackage{microtype}
\usepackage{graphicx}
\usepackage{subcaption}
\usepackage{booktabs} 

\usepackage{hyperref}

\usepackage{amsmath}
\usepackage{amssymb}
\usepackage{mathtools}
\usepackage{amsthm}
\usepackage{bm}
\usepackage{algorithm}
\usepackage{algorithmic}
\usepackage{graphicx}
\usepackage{wrapfig}
\usepackage[capitalize,noabbrev]{cleveref}

\theoremstyle{plain}
\newtheorem{theorem}{Theorem}[section]
\newtheorem{proposition}[theorem]{Proposition}

\theoremstyle{definition}
\newtheorem{definition}[theorem]{Definition}

\theoremstyle{remark}

\usepackage[textsize=tiny]{todonotes}

\usepackage{hyperref}
\usepackage{url}

\usepackage[utf8]{inputenc} 
\usepackage[T1]{fontenc}    
\usepackage{hyperref}       
\usepackage{url}            
\usepackage{booktabs}       
\usepackage{amsfonts}       
\usepackage{nicefrac}       
\usepackage{microtype}      
\usepackage{xcolor}         

\newcommand{\ie}[1]{\textit{i.e., }{#1}}
\newcommand{\eg}[1]{\textit{e.g., }{#1}}

\usepackage{booktabs}
\usepackage{multirow}
\usepackage{caption}

\title{HyFAD: Hybrid Time-Frequency Diffusion with Frequency-Aware Embedding for Time Series Imputation}

\author{Hongfan Gao, Wangmeng Shen, Bin Yang, Jilin Hu\thanks{Corresponding Author}\\
School of Data Science and Engineering\\
East China Normal University\\
\texttt{\{hf.gao,wmshen\}@stu.ecnu.edu.cn}\\
\texttt{\{byang,jlhu\}@dase.ecnu.edu.cn}
}

\begin{document}

\maketitle

\begin{abstract}
Diffusion models have demonstrated strong performance in time series modeling due to their ability to progressively capture complex data distributions through iterative denoising. However, existing approaches struggle with frequency-sensitive denoising, high-frequency reconstruction and balancing global trends with local dynamics. To address these limitations, we propose \textbf{HyFAD}, a \textbf{Hy}brid time-frequency \textbf{D}iffusion model with \textbf{F}requency-\textbf{A}ware embedding for time series imputation. Built upon the DDPM paradigm, HyFAD adopts a coupled time-frequency diffusion framework, in which the reverse denoising proceeds sequentially from the time domain to the frequency domain, enabling coarse-to-fine generation. Specifically, the time-domain diffusion process captures low-frequency global trends, while the frequency-domain diffusion process refines high-frequency spectral components. We further introduce a frequency-aware step embedding that exploits the relationship between diffusion steps and spectral components, providing step-dependent spectral guidance and facilitates more accurate band-wise reconstruction. Extensive experiments on multiple benchmark datasets demonstrate that HyFAD achieves state-of-the-art performance. Our source code is available at \url{https://github.com/hongfangao/HyFAD}.
\end{abstract}

\section{Introduction}

Diffusion models~\citep{DBLP:conf/nips/HoJA20,DBLP:conf/iclr/song21score,DBLP:conf/icml/Sohl-DicksteinW15} have emerged as a powerful class of generative frameworks for time series modeling, owing to their strong ability to approximate complex data distributions. By progressively denoising latent representations initialized from Gaussian noise, these models learn a reverse process that captures rich temporal dependencies. As a result, diffusion-based approaches have achieved state-of-the-art performance across a range of time series tasks, including imputation~\citep{DBLP:conf/nips/TashiroSSE21,DBLP:journals/tmlr/AlcarazS23}, forecasting~\citep{DBLP:conf/icml/RasulSSV21,DBLP:conf/nips/KolloviehABZ0023}, and generation~\citep{DBLP:conf/nips/ColettaGBV23,DBLP:conf/iclr/YuanQ24}.

Despite their success, existing diffusion models for time series still suffer from several notable limitations:
(i)~\textbf{Insufficient reconstruction of high-frequency details}~\citep{DBLP:conf/nips/GalibTL24,DBLP:conf/nips/YangSYC24,DBLP:journals/corr/abs-2505-11278}. In standard time-domain diffusion, each diffusion step injects isotropic Gaussian noise whose energy is uniformly distributed across frequencies. However, real-world time series usually exhibit highly imbalanced spectral energy distributions, with most energy concentrated in low-frequency components. This mismatch makes high-frequency reconstruction more challenging and limits the model's ability to recover fine-grained temporal details. 
(ii)~\textbf{Imbalanced modeling of global trends and local dynamics}~\citep{DBLP:conf/icml/CrabbeHSS24,DBLP:conf/iclr/ShenCK24,DBLP:conf/nips/NaimanBPAFA24}. Global trends and local dynamics often compete during reconstruction. Prioritizing long-term structures may lead to over-smoothed outputs, while emphasizing short-term fluctuations can disrupt global consistency. 
(iii)~\textbf{Frequency-insensitive modeling}~\citep{DBLP:conf/icml/WangYW0025,DBLP:conf/cvpr/YangZFW23}. Existing diffusion models typically use frequency-agnostic step embeddings and inject noise uniformly across frequency bands. As a consequence, they have limited ability to adapt their denoising behavior to the distinct recovery characteristics of low- and high-frequency components, which can lead to over-smoothed reconstructions when local variations are weak or sparsely observed.

Fig.~\ref{fig:intro} provides a representative imputation example on PhysioNet. The time-domain diffusion baseline, CSDI, captures the coarse level of the sequence but attenuates the sharp local variation in the highlighted region, resulting in an over-smoothed reconstruction. The spectrum computed over the same window further shows that the reconstructed signal underestimates middle- and high-frequency amplitudes compared with the ground truth. This example indicates that ignoring frequency dynamics can lead to observable reconstruction errors, motivating diffusion models that explicitly account for frequency-domain denoising.

\begin{figure}
    \centering
    \includegraphics[width=0.85\linewidth]{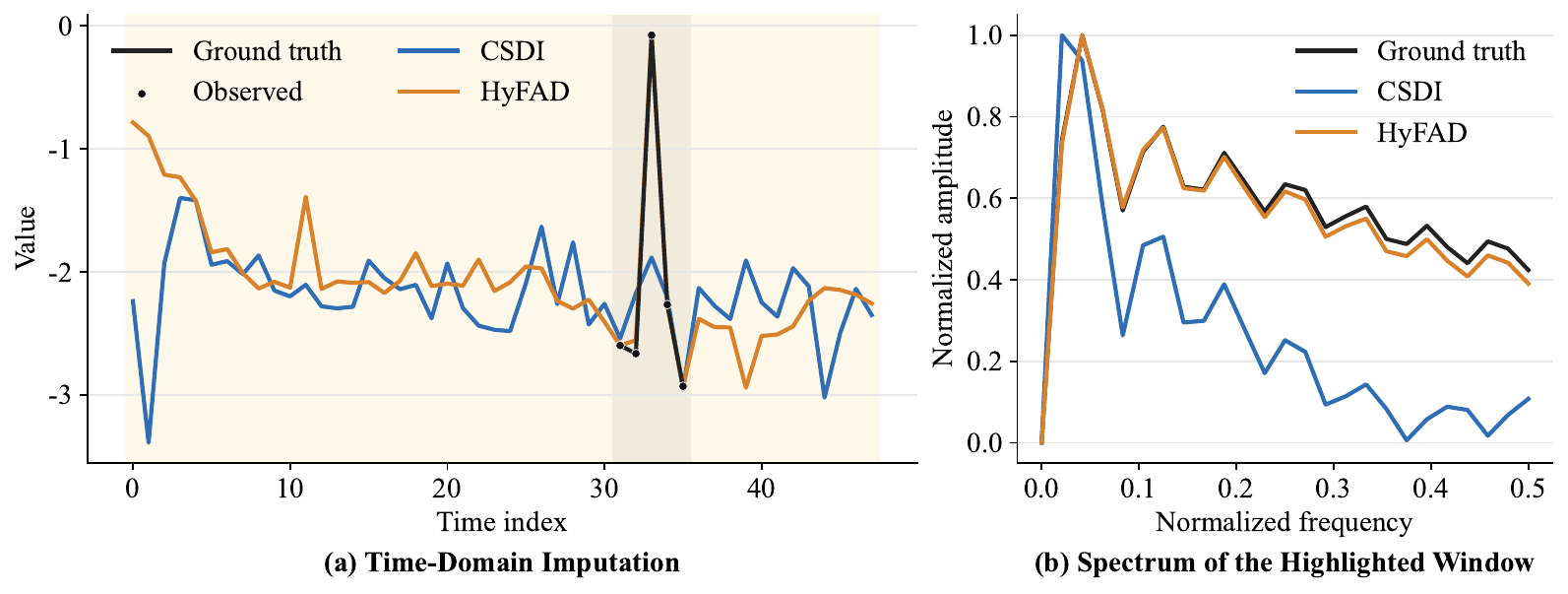}
    \caption{An imputation example from Channel 4 on PhysioNet. While CSDI captures coarse trends, it attenuates sharp local variations in the highlighted region. In contrast, HyFAD better preserves these local structures and matches the ground-truth spectrum, correcting CSDI's underestimation of mid/high-frequency amplitudes.}
    \label{fig:intro}
    \vspace{-7mm}
\end{figure}

In summary, these limitations can be attributed to two fundamental causes. 
(i)~\textbf{Time-domain-only noise modeling.} In most existing diffusion models, both noise injection and denoising are performed exclusively in the time domain. This design overlooks the imbalanced spectral energy distribution of real-world time series, where high-frequency components often have weaker energy and are therefore harder to recover. Incorporating frequency-domain modeling enables explicit characterization of band-wise spectral statistics, allowing high-frequency transients and local dynamics to be more effectively captured during denoising. 
(ii)~\textbf{Lack of frequency-specific attention in step embeddings.} Existing diffusion step embeddings do not incorporate frequency-dependent information and therefore fail to capture how different frequency bands should be emphasized or suppressed at different denoising stages. This prevents the model from adaptively allocating capacity between low- and high-frequency components, exacerbating the loss of fine-grained details and weakening multi-scale representation.

To address these challenges, we propose \textbf{HyFAD}, a \textbf{Hy}brid time-frequency \textbf{D}iffusion model with \textbf{F}requency-\textbf{A}ware embedding for time series imputation. Unlike conventional time-domain diffusion approaches, HyFAD incorporates frequency-domain dynamics into the diffusion process, enabling explicit differentiation across spectral bands and improving high-frequency reconstruction. Specifically, HyFAD formulates a coupled time-frequency diffusion process, where frequency-domain and time-domain noise are injected sequentially in the forward process and removed in a mirrored order during reverse denoising. Furthermore, we analyze the relationship between diffusion steps and spectral components and introduce a frequency-aware step embedding that provides step-dependent spectral guidance, allowing the model to focus on appropriate frequency bands throughout denoising.

Our contributions are summarized as follows:
\begin{enumerate}
    \item We propose HyFAD, a coupled time-frequency diffusion framework for time series imputation. The proposed forward and reverse processes explicitly model both time-domain and frequency-domain noise within a unified DDPM framework, enabling coarse-to-fine reconstruction.
    \item We introduce a frequency-aware diffusion step embedding that provides step-dependent spectral guidance by modeling the correspondence between diffusion steps and frequency bands, allowing adaptive focus on appropriate spectral components during denoising.
    \item We conduct extensive experiments on multiple real-world benchmark datasets, where HyFAD consistently achieves state-of-the-art performance, demonstrating the effectiveness of the proposed method.
\end{enumerate}

\section{Preliminaries}

\subsection{Discrete Fourier Transform}
\label{sec:dft}
Given a real-valued discrete time series $\mathbf{x} = [x_0, \ldots, x_{L-1}] \in \mathbb{R}^L$ with length $L$, its unitary Discrete Fourier Transform (DFT), denoted by $\tilde{\mathbf{x}} \in \mathbb{C}^L$, is defined as:
\begin{equation}
    \tilde{\mathbf{x}}_k = \frac{1}{\sqrt{L}}\sum_{\tau=0}^{L-1}\mathbf{x}_{\tau}\exp\left(-
    \frac{k2\pi i}{L}\tau\right),
\end{equation}
where $k$ is the frequency index, $\tau$ is the time index, and $i$ denotes the imaginary unit. As $\mathbf{x}$ is real-valued, $\tilde{\mathbf{x}}$ exhibits conjugate symmetry \ie{$\tilde{\mathbf{x}}_k = \tilde{\mathbf{x}}_{L-k}^*$}, making almost half of the coefficients redundant.
To obtain a compact representation that preserves isometry (\ie satisfying Parseval's identity) in the real domain, we define an \emph{orthonormal real DFT}~(rDFT) transform.

Specifically, we exploit the spectral symmetry to discard the redundant negative frequencies. To offset the energy reduction caused by this truncation, we apply a scaling factor of $\sqrt{2}$ to the real and imaginary components of the intermediate frequencies. The elements of the transformed vector $\mathbf{y}$ are arranged as follows: 
\begin{equation}
    \mathbf{y}_m = \begin{cases}
        \Re(\tilde{\mathbf{x}}_0)&m=0\\
        \Re(\tilde{\mathbf{x}}_m)&m=\frac{L}{2}~(\text{L is even})\\
        \sqrt{2}\Re(\tilde{\mathbf{x}}_m)&0<m<\lceil\frac{L}{2}\rceil\\
        \sqrt{2}\Im(\tilde{\mathbf{x}}_{k(m)})&m\ge N_r,
    \end{cases}
\end{equation}
where $\Re(\tilde{\mathbf{x}})$ and $\Im(\tilde{\mathbf{x}})$ denote the real and imaginary parts of $\mathbf{x}$ and
$\mathbf{y} = [y_0,\ldots,y_{L-1}]^\top \in \mathbb{R}^L$ is the real-valued frequency representation constructed from $\Re(\tilde{\mathbf{x}})$ and $\Im(\tilde{\mathbf{x}})$.
The index $m=0,\ldots,L-1$ denotes the entry index of $\mathbf{y}$, $k(m)=m-N_r+1$ maps entries in the imaginary-part block of $\mathbf{y}$ to the corresponding positive-frequency index and $N_r = \lfloor\frac{L}{2}\rfloor$ is the number of frequency bins in the complex spectrum.

It is worth noticing that rDFT is a linear transformation, \ie{$\mathbf{y} = \mathbf{W}\mathbf{x}$}, where $\mathbf{W} \in \mathbb{R}^{L \times L}$ is the transformation matrix and orthogonal~(\ie{$\mathbf{W}\mathbf{W}^T = \mathbf{I}$}). This implies that the rDFT operation is invertible, allowing for reconstructing the time-domain representation $\mathbf{x}$ from the frequency-domain representation $\mathbf{y}$ with $\mathbf{x}=\mathbf{W}^T\mathbf{y}$.

\subsection{Time Series Diffusion in the Frequency Domain}
\label{sec:tsdf}
While most existing diffusion-based approaches~\citep{DBLP:conf/nips/WangZQ0W0L23,DBLP:journals/corr/abs-2106-10121} 
for time series focus exclusively on modeling in the time domain, \citep{DBLP:conf/icml/CrabbeHSS24} explores 
the diffusion process in the frequency domain by introducing a noise scaling matrix 
$\Lambda$ to the frequency-domain diffusion process under DFT. $\Lambda$ plays to essential roles: (i) Rescaling spectral coefficients to make frequency-domain noise isotropic across all frequency components, (ii) Ensuring energy consistency between the time and frequency domains under the frequency domain transform. 
For an input sequence $\mathbf{x}\in\mathbb{R}^{K\times L}$, the diagonal noise scaling matrix 
$\Lambda\in\mathbb{R}^{L\times L}$ is defined as :
\begin{equation}
    \Lambda=\begin{cases}
    1& n=1 \text{ or $n=\frac{L}{2}$}~(L~\text{is even}) \\
    \frac{1}{\sqrt{2}}&\text{otherwise}
    \end{cases}
    \label{eq:g}
\end{equation}
Different from \citep{DBLP:conf/icml/CrabbeHSS24}, we set $\Lambda=\mathbf{I}$ since rDFT in Sec.\ref{sec:dft} is employed in our implementations instead of DFT for isotropic noise and energy consistency.
\subsection{Problem Definition and Notations}
\begin{definition}[\textbf{Time Series with Missing Value}]
\label{def:incts}
A time series with missing values is defined as $\tilde{\mathbf{X}}=(\mathbf{X}, \mathbf{M}, \mathbf{T})$, where $\mathbf{X} \in \mathbb{R}^{K \times L}$ is the observation matrix with $K$ observations at a time, which are ordered along $L$ time intervals chronologically; $\mathbf{M}\in \mathbb{R}^{K \times L}$ is an indicator matrix that indicates whether the observation at $(i,j)$ in $\mathbf{X}$ is missing or not. We use $M^{\mathrm{obs}} \in \{0,1\}^{K \times L}$ to denote the observation mask, where $M^{\mathrm{obs}}_{i,j}=1$ indicates that $X_{i,j}$ is observed. We use $M^{\mathrm{tar}}$ to denote the target mask for evaluation or training, where $M^{\mathrm{tar}}_{i,j}=1$ indicates an imputation target.
 
\end{definition}
\textbf{Problem Statement}~\textbf{(Time Series Imputation).}
Given a time series with missing value $\tilde{\mathbf{X}}=(\mathbf{X}, \mathbf{M}, \mathbf{T})$, the goal of time series imputation is to learn an imputation function $\mathcal{M}_{\theta}$, such that 
\begin{equation}
    \bar{\mathbf{X}} = \mathcal{M}_{\theta}(\tilde{\mathbf{X}}),
\end{equation}
where $\bar{\mathbf{X}}\in \mathbb{R}^{K \times L}$ is the imputed time series, $\bar{\mathbf{X}}_{i,j}$ denotes the imputation output if $\mathbf{M}_{i,j} = 1 $, otherwise $\bar{\mathbf{X}}_{i,j} = \tilde{{\mathbf{X}}}_{i,j}$. 
\label{sec:nota}

\textbf{Notations.}~We adopt the following notations. We use superscripts to denote the domain of the data, and subscripts to denote denoising steps, \eg{$\mathbf{x}_k^f$ refers to the data at diffusion step $k$ in the frequency domain and $\boldsymbol{\epsilon}_k^t$ refers to the noise at diffusion step $k$ in the time domain}. $\beta$ is a predefined scheduler over $T$ denoising steps and $\alpha_{k} = 1-\beta_k,k\in\{1,2,\cdots,T\}$. We denote $\bar{\alpha}_{i:j}=\Pi_{k=i}^j\alpha_k$ and if $i=1$, we omit the lower index $i$ and use $\bar{\alpha}_j$ for short. Specially, we define $\bar{\alpha}_{0:j}=1$ for arbitrary $j$. Besides, for time series data, we following the same notations as CSDI, \ie{$\mathbf{X}^{\text{ta}},\mathbf{X}^{\text{obs}}$ denotes the imputation targets and ground truth values and $\mathbf{M}^{\text{cond}}$ denotes the condition matrix.}

\section{Methodology}
\begin{figure*}[tbp]
    \centering
    \includegraphics[width=0.8\linewidth]{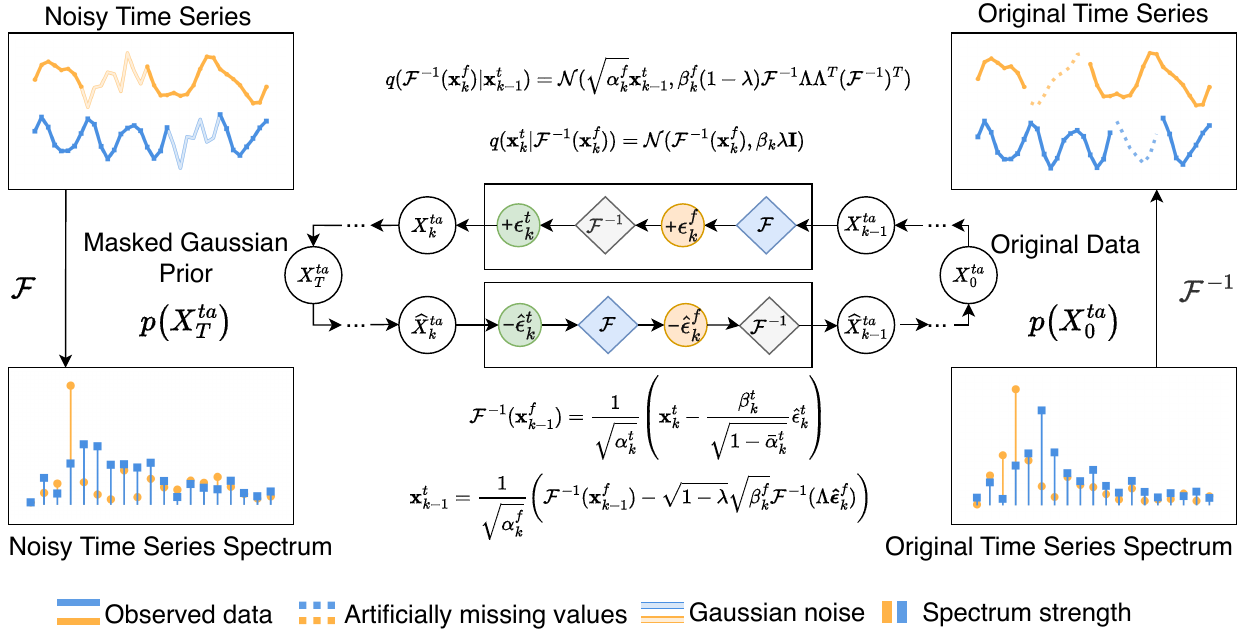}
    \caption{Noise injection and denoising process in HyFAD. At each forward step, frequency-domain noise is injected first and then time-domain noise. In the reverse process, time-domain noise is removed first, followed by frequency-domain noise.}
    \label{fig:forrev}
    \vspace{-5mm}
\end{figure*}
\subsection{Hybrid Time-Frequency Diffusion Model}
\subsubsection{Forward Process of HyFAD}
\label{sec:forward}
The forward process of HyFAD consists of two successive parts in a single  step,
\ie{noise is first injected in the frequency domain and subsequently in the time domain.}
By iteratively performing these two successive noise injection procedures, the original data distribution $p(x_0)$ is transformed into a target Gaussian distribution. Given frequency transform $\mathcal{F}$, we first map the data into corresponding domain via $\mathcal{F}$, and then inject random gaussian noise $\boldsymbol{\epsilon}_k^f\sim\mathcal{N}(0,\mathbf{I})$ into this frequency domain representation following a DDPM-style forward diffusion process, which is formulated as:
\begin{equation}
    \mathbf{x}_k^f = \sqrt{\alpha_k^f}\mathcal{F}(\mathbf{x}_{k-1}^t) + \sqrt{\beta_k^f}\sqrt{1-\lambda}(\Lambda\boldsymbol{\epsilon}_k^f),
    \label{eq:addnoisef1}
\end{equation}
where $\lambda$ is the coefficient for balancing variance between frequency 
and time domain noise and $\boldsymbol{\epsilon}_k^f\sim\mathcal{N}(0,\mathbf{I})$.
Then the frequency-domain noisy sample $\mathbf{x}_{k}^f$ in Eq.\ref{eq:addnoisef1} is transformed back to time domain:
\begin{equation}
    \mathcal{F}^{-1}(\mathbf{x}_k^f) = \sqrt{\alpha_k^f}\mathbf{x}_{k-1}^t + \sqrt{\beta_k^f}\sqrt{1-\lambda}\mathcal{F}^{-1}(\Lambda\boldsymbol{\epsilon}_k^f),
    \label{eq:addnoisef2}
\end{equation}
where $\boldsymbol{\epsilon}_k^f\sim\mathcal{N}(0,\mathbf{I})$ and we can obtain the posterior
$
    q(\mathcal{F}^{-1}(\mathbf{x}_{k}^f)\vert \mathbf{x}_{k-1}^t)\coloneqq\mathcal{N}(\sqrt{\alpha_k^f}\mathbf{x}_{k-1}^t,\beta_k^f(1-\lambda)\mathcal{F}^{-1}\Lambda\Lambda^T(\mathcal{F}^{-1})^{T})
$.

Then another \textbf{independent} \textbf{gaussian} \textbf{noise} $\boldsymbol{\epsilon}_k^t\sim\mathcal{N}(0,\mathbf{I})$ is injected to $\mathcal{F}^{-1}(\mathbf{x}_k^f)$, \ie{noise is injected in the time domain}:
\begin{equation}
    \mathbf{x}_k^t = \sqrt{\alpha_k^t}\mathcal{F}^{-1}(\mathbf{x}_k^f) + \sqrt{\beta_k^t}\sqrt{\lambda}\boldsymbol{\epsilon}_k^t.
    \label{eq:addnoiset}
\end{equation}
Thus, the posterior can be written as $q(\mathbf{x}_k^t\vert \mathcal{F}^{-1}(\mathbf{x}_{k}^f))\coloneqq\mathcal{N}(\sqrt{\alpha_k^t}\mathcal{F}^{-1}(\mathbf{x}_{k}^f),\beta_k^t\lambda\mathbf{I})$.
With this, the noise injection loop $k$ is completed. We can formulate the relationship between $\mathbf{x}_{k}^t$ and $\mathbf{x}_{k-1}^t$ as:
\begin{equation}
    \mathbf{x}_k^t = \sqrt{\alpha_k^t\alpha_k^f}\mathbf{x}_{k-1}^t + \sqrt{\alpha_k^t(1-\alpha_k^f)}\sqrt{1-\lambda}\mathcal{F}^{-1}(\Lambda\boldsymbol{\epsilon}_k^f)+\sqrt{1-\alpha_k^t}\sqrt{\lambda}\boldsymbol{\epsilon}_k^t.
    \label{eq:addnoiseall}
\end{equation}
By iteratively preforming the noise injection process according to Eq.\ref{eq:addnoisef2}, Eq.\ref{eq:addnoiset} and Eq.\ref{eq:addnoiseall}, the closed-form relationship between $\mathbf{x}_k^t$ and $\mathbf{x}_0^t$ can be derived as:
\begin{equation}
    \mathbf{x}_k^t = \sqrt{\bar{\alpha}_k^t\bar{\alpha}_k^f}\mathbf{x}_0^t + \sqrt{1-\lambda}\sum_{s=1}^k\sqrt{\beta_{s}^f}\sqrt{\frac{\bar{\alpha}_k^t\bar{\alpha}_k^f}{\bar{\alpha}_{s-1}^t\bar{\alpha}_s^f}}\mathcal{F}^{-1}(\Lambda\boldsymbol{\epsilon}_s^f) + \sqrt{\lambda}\sum_{s=1}^k\sqrt{\beta_s^t}\sqrt{\frac{\bar{\alpha}_k^t}{\bar{\alpha}_s^t}}\boldsymbol{\epsilon}_s^t.
    \label{eq:addnoisetotal}
\end{equation}
According to Eq.\ref{eq:addnoisetotal}, for any diffusion step $k$, $\mathbf{x}_k^t$ can be expressed as a linear transformation of the original signal $\mathbf{x}_0^t$ with two groups of independent gaussian noise $\boldsymbol{\epsilon}_s^t$ and $\boldsymbol{\epsilon}_s^f$, $s\in\{1,2,\cdots,k\}$. Therefore, the posterior $q(\mathbf{x}_k^t\vert\mathbf{x}_0^t)$ remains gaussian, \ie{}
\begin{equation}
    q(\mathbf{x}_k^t\vert\mathbf{x}_0^t) \coloneqq \mathcal{N}(\boldsymbol{\mu}_k,\boldsymbol{\Sigma}_k),
    \label{eq:posterior}
\end{equation}
where 
$\boldsymbol{\mu}_k = \sqrt{\bar{\alpha}_k^t\bar{\alpha}_k^f}\mathbf{x}_0^t
    \label{eq:postmean}$
and 
$
\boldsymbol{\Sigma}_k = (1-\lambda)\sum_{s=1}^k\beta_s^f\frac{\bar{\alpha}_k^t}{\bar{\alpha}_{s-1}^t}\frac{\bar{\alpha}_k^f}{\bar{\alpha}_s^f}\left(\mathcal{F}^{-1}\Lambda\Lambda^T(\mathcal{F}^{-1})^T\right) +\lambda\sum_{s=1}^k\beta_s^t\frac{\bar{\alpha}_k^t}{\bar{\alpha}_s^t}\mathbf{I}.
\label{eq:postvariance}
$
\subsubsection{Reverse Process of HyFAD}
\label{sec:reverse}
To progressively reconstruct the original data from the diffused samples, 
we construct a reverse Markov chain that is consistent with the forward noise injection process. Following the DDPM formulation, we parameterize the reverse distribution $p_{\theta}(\mathbf{x}_{k-1}^t\vert\mathbf{x}_{k}^t)$ using a Gaussian approximation, aiming to model the true posterior $q(\mathbf{x}_{k-1}^t\vert \mathbf{x}_{k}^t,\mathbf{x}_0^t)$. This enables closed-form stepwise sampling throughout the denoising trajectory. Since our method introduces a combination of time-domain noise and frequency-domain noise~(Eq.\ref{eq:addnoiseall}) at each diffusion step, the reverse process must explicitly account for the structure of this composite noise to ensure faithful reconstruction of original data.

To improve sampling stability, we adopt a DDIM-style sampling strategy~\citep{DBLP:conf/iclr/SongME21}, which follows a deterministic trajectory in the reverse process. We initialize from a Gaussian prior that is consistent with the terminal distribution of the forward diffusion. As in Eq.\ref{eq:posterior}, the sampling process starts from 
\begin{equation}
\mathbf{x}_{T}^t=\sqrt{\lambda}\boldsymbol{\epsilon}^t+\sqrt{1-\lambda}\mathcal{F}^{-1}(\Lambda\boldsymbol{\epsilon}^f),\boldsymbol{\epsilon}^t,\boldsymbol{\epsilon}^f\sim\mathcal{N}(0,\mathbf{I}). 
\end{equation}
At each step $k$, the model predicts the time domain noise $\hat{\boldsymbol{\epsilon}}_{k}^t$ and
frequency domain noise $\hat{\boldsymbol{\epsilon}}_{k}^f$. Reverse denoising is then performed in a step-wise, coupled manner within the same diffusion loop. First, the time-domain noise in $\mathbf{x}_{k}^t$ is removed according to:
\begin{equation}
    \mathcal{F}^{-1}(\mathbf{x}_{k-1}^f) = \frac{1}{\sqrt{\alpha_{k}^t}}\left(\mathbf{x}_{k}^t-\frac{\beta_{k}^t}{\sqrt{1-\bar{\alpha}_{k}^t}}\hat{\boldsymbol{\epsilon}}_{k}^t\right). 
    \label{eq:timedenoising}
\end{equation}
Then the frequency-domain noise in $\mathcal{F}^{-1}(\mathbf{x}_{k-1}^f)$ is removed for $\mathbf{x}_{k-1}^t$:
\begin{equation}
    \mathbf{x}_{k-1}^t = \frac{1}{\sqrt{\alpha_{k}^f}}\left(\mathcal{F}^{-1}(\mathbf{x}_{k-1}^f)-\sqrt{1-\lambda}\sqrt{\beta_{k}^f}\mathcal{F}^{-1}(\Lambda\boldsymbol{\hat{\epsilon}}_k^f)\right).
    \label{eq:freqdenosing}
\end{equation}

Eq.\ref{eq:timedenoising} and Eq.\ref{eq:freqdenosing} define distinct update rules for time and frequency domain denoising process.
To summarize, the step-wise denoising strategy \ie{time-to-frequency denoising}, 
serves as the mirror counterpart of the forward diffusion process (frequency-to-time noise injection). 
Empirically, the time-domain branch is effective in stabilizing global structures and low-frequency trends, while the frequency-domain branch focuses on high-frequency component, together forming a coarse-to-fine denoising mechanism. The forward and reverse process is illustrated in Fig.\ref{fig:forrev}.
\subsubsection{Loss Function}
As stated in Sec.\ref{sec:forward} and Sec.\ref{sec:reverse}, the overall noise injection process consists
of two successive sub-processes, \ie{time and frequency domain noise injection}. We adopt the commonly used noise 
estimation approach in diffusion models. 
The loss function consists of three components: 
(i) time domain noise estimation loss, \ie{the $l_2$ distance between the estimated time domain noise and forward time domain noise}, 
(ii) frequency domain noise estimation loss, \ie{the $l_2$ distance between the estimated frequency domain noise and forward frequency domain noise}.
(iii) consistency loss, \ie{the $l_2$ distance between total noise from step $k$ to step $k-1$}. This consistency term mitigates potential mismatches arising from estimating the two branches separately and discourages arbitrary residual reallocation across branches.
Accordingly, the overall loss is defined as:
\begin{equation}
    \mathcal{L}^{d} = \mathbb{E}_{\mathbf{x}_0\sim q(\mathbf{x}_0),k}(\Vert \boldsymbol{\epsilon}_k^t-\boldsymbol{\hat{\epsilon}}_k^t\Vert_2^2 +\Vert \boldsymbol{\epsilon}_k^f-\mathcal{F}^{-1}(\Lambda\boldsymbol{\hat{\epsilon}}_k^f)\Vert_2^2 +\omega\Vert(\boldsymbol{\epsilon}_k^t+\boldsymbol{\epsilon}_k^f)-(\hat{\boldsymbol{\epsilon}}_k^t+\mathcal{F}^{-1}(\Lambda\hat{\boldsymbol{\epsilon}}_k^f))\Vert_2^2),
\label{eq:loss}
\end{equation}
where $\omega$ is the weight for consistency loss. In addition, since our model is specifically designed for time series imputation task, we focus solely on the reconstruction error over masked regions during training, \ie{$\mathcal{L}= \mathbf{M}^t\odot\mathcal{L}^{d}$}. The detailed training and sampling algorithm is presented in Sec.\ref{sec:alg}.

\subsection{Frequency-aware Diffusion Embedding}
\label{sec:fade}
Considering the relationship between denoising steps and the frequency components of the data, we present the following proposition~\citep{DBLP:conf/icml/Ning0SJL0CSE25,DBLP:journals/corr/abs-2501-18232,DBLP:conf/cvpr/QianCPLYSM24}:
\begin{proposition}[Frequency Components in the Diffusion Process]
\label{prop:freq}
Consider the forward process where the noise is isotropic and has cumulative spectral energy $N_t$ at diffusion step $t$. Partition the frequency axis into ordered bands $\{B_b\}_{b=1}^B$ from low to high frequency, and let $E_b$ denote the band-averaged clean-signal energy in band $B_b$. If the signal satisfies the relative low-pass condition $E_{b_1} \ge E_{b_2}$ for any $b_1<b_2$, then for any SNR threshold $\gamma>0$, the threshold time
$
\tau_b(\gamma)
= \inf \left\{t \in [0,T]: \frac{E_b}{N_t} \le \gamma \right\}
$
satisfies
$
\tau_{b_2}(\gamma) \le \tau_{b_1}(\gamma),
\forall b_1<b_2.
$
That is, higher-frequency bands reach the same degradation threshold no later than lower-frequency bands during the forward process. Equivalently, the reverse process tends to recover low-frequency structures before progressively restoring high-frequency details.
\end{proposition}
\begin{proof}
    The proof is in Appendix.\ref{proof:freq}.
\end{proof}

Motivated by Prop.~\ref{prop:freq}, we propose a frequency-aware diffusion embedding to capture the stage-dependent recoverability of different frequency components during denoising. As the noise level changes across diffusion steps, different frequency bands should be emphasized differently. To this end, the proposed embedding integrates three components: \textbf{noise-aware frequency gate}, \textbf{stage-aware frequency schedule}, and \textbf{missing-aware band reliability}.

\textbf{Noise-aware Frequency Gate}.~Given the accumulated diffusion noise, whose expected spectral energy satisfies $\mathbb{E}[\vert \hat{\epsilon}_t(\omega)\vert^2] = \int_0^t \vert g(s)\vert^2\mathrm{d}s$, the diffusion step $t$ corresponds to a frequency-agnostic noise scale $N_t=\int_0^t \vert g(s)\vert^2\mathrm{d}s$, which decreases monotonically in the reverse process. Since different bands have different signal-to-noise ratios, treating all frequency components uniformly is suboptimal. We therefore introduce a noise-aware gate, with calibration parameters $\gamma$ and $s$ controlling the gating threshold and its scale:
\begin{equation}
    \hat{g}_b(t) = \sigma\left(\frac{\log P_b^{\mathrm{sig}} -\log (\gamma N_ts)}{\tau}\right),
\end{equation}
where $P_b^{\mathrm{sig}}$ is the spectral energy of frequency band $b$ computed from the signal proxy detailed in Sec.~\ref{sec:mafm}, and $\tau$ controls the sharpness of the gate. To avoid numerical issues, we set a minimum value $g_{\min}$:
\begin{equation}
    g_b(t) = g_{\min} + (1-g_{\min})\hat{g}_{b}(t).
\end{equation}

\textbf{Stage-aware Frequency Schedule}.~Following Prop.~\ref{prop:freq}, we further encode the progressive low-to-high frequency recovery process into the diffusion embedding. For each band $b$, the stage-dependent scheduling weight $S_b(t)$ measures its relative importance at diffusion step $t$:
\begin{equation}
S_b(t) = \exp\left(-\left(\frac{\omega_b}{c(t)}\right)^p\right),
\end{equation}
where $\omega_b$ is the normalized frequency coordinate, $c(t)$ gradually relaxes the effective bandwidth as denoising progresses, and $p>0$ is a constant. Unlike the noise-aware gate, $S_b(t)$ depends only on $t$ and $\omega_b$, providing a stage-dependent frequency prior independent of actual spectral energy or noise magnitude.

\textbf{Band Reliability}.~For time series imputation, missing values can introduce spectral artifacts and weaken the reliability of frequency-domain modeling. We therefore define a missing-aware reliability weight for each frequency band:
\begin{equation}
    R_b = \frac{P_b^{\mathrm{sig}}}{P_b^{\mathrm{sig}}+\kappa P_b^{\mathrm{mask}}},
\end{equation}
where $P_b^{\mathrm{mask}}$ denotes the spectral energy induced by the missing pattern in band $b$, and $\kappa$ is a constant.

Finally, we combine the noise-aware gate, stage-aware schedule, and band reliability through element-wise multiplication, and project the result into the embedding space:
\begin{equation}
    \mathbf{g}(t) = \mathrm{Proj}\left(g(t)\odot S(t)\odot R,\, \frac{d}{2}\right),
    \label{eq:proj}
\end{equation}
where $\mathrm{Proj}(\cdot,\frac{d}{2})$ resamples $\lfloor\frac{L}{2}\rfloor + 1$ frequency bins to $\frac{d}{2}$ embedding dimensions. The resulting gating vector modulates sinusoidal basis functions to form the frequency-aware diffusion embedding:
\begin{equation}
    \mathbf{e}(t)
    =
    \left[
        \sin(t\,\mathbf{f}_d)\odot \mathbf{g}(t),
        \;
        \cos(t\,\mathbf{f}_d)\odot \mathbf{g}(t)
    \right],
    \label{eq:grid}
\end{equation}
where $\mathbf f_d$ denotes a linearly spaced frequency grid and $t$ is the diffusion step. Details are provided in Appendix~\ref{sec:appnafe}.
\subsection{Model Details}
\subsubsection{Missing-aware Frequency Modeling}
\label{sec:mafm}
Direct frequency transformation of the masked series $\tilde{\mathbf{X}}=\mathbf{M}\odot \mathbf{X}$ often induces spectral leakage, distorting the true spectral morphology. To mitigate this, we introduce a Signal Proxy mechanism that serves as a diffusion-synchronized spectral reference. Designed to ensure observational consistency and reliable estimation, the proxy evolves with the diffusion process to reflect the changing signal-to-noise ratio. Formally, we define the signal proxy $\tilde{\mathbf{x}}_{\mathrm{proxy}}$ at step $t$ as:
\begin{equation}
\tilde{\mathbf{x}}_{\mathrm{proxy}} = \mathbf M\odot \mathbf x_{\mathrm{obs}} + (1-\mathbf M)\odot \mathbf x_t
\end{equation}
where $\mathbf{x}_{\mathrm{obs}}$ denotes the observed values and $\mathbf{x}_t$ is the temporal latent variable. This formulation preserves the authentic signal structure at observed locations while populating missing intervals with state estimates consistent with the current diffusion phase.  This diffusion-aware mechanism effectively minimizes spectral interference, enabling robust frequency representation learning.

\subsubsection{Model Architecture}
We present the details of the denoising network.We adopt a cascaded dual-branch network as the core denoising architecture, 
where both branches share the same network structure. Considering the intra-channel 
and inter-channel dependencies in multivariate time series, the temporal-feature transformer architecture from 
CSDI~\citep{DBLP:conf/nips/TashiroSSE21} is adopted as the diffusion backbone.

Specifically, our model takes the noisy data, conditional information and step embedding as input. 
The input is first processed by a temporal attention module and a feature attention module for feature extraction,
followed by a linear projection to estimate time-domain noise $\hat{\boldsymbol{\epsilon}}_{t}$ 
To ensure that the frequency-domain branch focuses exclusively on estimating frequency-domain noise, 
we transform the output of the time-domain branch into the frequency domain via a given transformation $\mathcal{F}$ and use it as the input to the frequency-domain branch. 
The frequency-domain branch is symmetric to the time-domain branch, except that the step embedding is replaced with the frequency-aware diffusion embedding introduced 
in Sec.~\ref{sec:fade}. Finally, the frequency-domain output is mapped back to the time domain via $\mathcal{F}^{-1}$.
Overall, the model produces two outputs: the estimated time-domain noise $\hat{\boldsymbol{\epsilon}}_{t}$ and frequency-domain noise $\hat{\boldsymbol{\epsilon}}_{f}$. The details of our model architecture is in the Appendix~\ref{sec:archdetails}.

\begin{table*}[tbp]
\centering
\caption{Imputation performance on PhysioNet and Air Quality datasets. 
Best results are in \textbf{bold}; second-best are \underline{underlined}}
\label{tab:results}
\resizebox{0.75\linewidth}{!}{
\begin{sc}
\begin{tabular}{lcccccccc}
\hline
\multirow{2}{*}{Method} & \multicolumn{2}{c}{PhysioNet 10\%} & \multicolumn{2}{c}{PhysioNet 50\%} & \multicolumn{2}{c}{PhysioNet 90\%} & \multicolumn{2}{c}{Air Quality}    \\ \cline{2-9} 
                        & MAE              & RMSE            & MAE              & RMSE            & MAE              & RMSE            & MAE             & RMSE             \\ \hline
Mean                    & 0.714            & 1.035           & 0.711            & 1.091           & 0.710            & 1.097           & 50.685          & 66.558           \\
Lerp                    & 0.372            & 0.708           & 0.417            & 0.840           & 0.565            & 0.993           & 15.363          & 27.658           \\
BRITS                   & 0.278            & 0.693           & 0.385            & 0.833           & 0.560            & 0.975           & 16.519          & 26.775           \\
GPVAE                   & 0.469            & 0.783           & 0.521            & 0.907           & 0.642            & 1.038           & 23.941          & 40.586           \\
SSGAN                   & 0.323            & 0.662           & 0.449            & 0.852           & 0.670            & 1.060           & 32.999          & 48.951           \\
TimesNet                & 0.375            & 0.690           & 0.453            & 0.840           & 0.642            & 1.031           & 22.685          & 39.336           \\
CSDI                    & 0.215            & 0.491           & 0.307            & 0.673           & \underline{0.492}            & 0.851           & 9.347           & 18.713           \\
SAITS                   & 0.232            & 0.583           & 0.315            & 0.735           & 0.565            & 0.971           & 15.424          & 30.558           \\
ModernTCN               & 0.351            & 0.697           & 0.440            & 0.803           & 0.647            & 1.026           & 24.089          & 40.052           \\
LSCD                    & \underline{0.212}  & \underline{0.457} & \underline{0.305}            & \underline{0.658}           & \underline{0.492}            & \underline{0.845}           & \underline{9.340}  & \underline{18.270} \\
HyFAD(Ours)             & \textbf{0.206}   & \textbf{0.423}  & \textbf{0.288}   & \textbf{0.622}  & \textbf{0.448}  & \textbf{0.782} & \textbf{9.296} & \textbf{17.730}  \\ \hline
\end{tabular}
\end{sc}}
\vspace{-5mm}
\end{table*}

\section{Experiments}
\label{sec:exp}

\textbf{Experimental Settings}.~All experiments are implemented in Python 3.12 with Pytorch~\citep{DBLP:conf/nips/PaszkeGMLBCKLGA19} on a single Nvidia RTX 3090 GPU. We use Adam~\citep{DBLP:journals/corr/KingmaB14} with an initial learning rate of $1\times10^{-3}$ and a multi-step scheduler that decays the learning rate by 10 at 75\% and 90\% of the total epochs. Quadratic noise schedules are adopted for both time- and frequency-domain diffusion, and the frequency transformation $\mathcal{F}$ is set to rDFT as described in Sec.~\ref{sec:dft}.

\textbf{Datasets and Evaluation Metrics}.~We evaluate HyFAD on two widely used imputation benchmarks: PhysioNet~\citep{silva2012predicting} and Air Quality~\citep{DBLP:conf/ijcai/YiZZL16}. We report MAE and RMSE between the imputed values and ground truth at missing positions. More details are provided in Appendix~\ref{sec:dadetails}.

\textbf{Baselines}.~We compare HyFAD with classical methods~(Mean and Lerp), deep learning approaches~(BRITS~\citep{DBLP:conf/nips/CaoWLZLL18}, SAITS~\citep{DBLP:journals/eswa/DuCL23}, TimesNet~\citep{DBLP:conf/iclr/WuHLZ0L23}), deep generative models~(GP-VAE~\citep{DBLP:conf/aistats/FortuinBRM20}, SSGAN~\citep{DBLP:conf/aaai/MiaoWWGMY21}, CSDI~\citep{DBLP:conf/nips/TashiroSSE21}, LSCD~\citep{DBLP:conf/icml/FonsSEFVV25}), and a time-series foundation model~(ModernTCN~\citep{DBLP:conf/iclr/LuoW24}).

\subsection{Time Series Imputation Results}
We evaluate HyFAD under different missing-data scenarios on two real-world datasets, \ie{PhysioNet and Air Quality}. Table~\ref{tab:results} reports the quantitative results, with visualizations provided in Appendix~\ref{sec:appendetails}.

As shown in Table~\ref{tab:results}, HyFAD achieves the best performance across all missing rates on PhysioNet. Compared with the second-best baseline, it reduces MAE by 2.8\%, 5.6\%, and 8.9\% under 10\%, 50\%, and 90\% missingness, respectively, while also consistently lowering RMSE. On Air Quality, HyFAD also obtains the best MAE and RMSE, demonstrating its effectiveness across different real-world time-series scenarios. Please refer to Appendix.\ref{sec:pinterval} for more detailed comparison.
\subsection{Ablation Studies and Parameter Analysis}
To validate the effectiveness of the proposed frequency-aware diffusion embedding, we replace it with the diffusion embedding in \citep{DBLP:conf/nips/TashiroSSE21}. We further analyze several key hyperparameters, including the terminal values of the noise schedules in time- and frequency-domain diffusion, \ie{$\beta_{\mathrm{end}}^t$ and $\beta_{\mathrm{end}}^f$}, the balancing parameter $\lambda$ that controls the relative strengths of the two diffusion processes, and $\omega$, the weight of the consistency loss.


As shown in Table~\ref{tab:abl}, the proposed Frequency-Aware Diffusion Embedding consistently outperforms the standard diffusion embedding across all missing rates on PhysioNet. In addition, since CSDI can be viewed as the time-domain-only counterpart of HyFAD, the consistent improvements over CSDI in Table~\ref{tab:results} further validate the necessity of hybrid time-frequency modeling. These results support the core design of HyFAD: the time branch preserves global structure, while the frequency branch complements local dynamics and high-frequency details.

For noise allocation, the best performance is achieved when the temporal process is dominant ($\lambda=0.75$). The performance drops at $\lambda=0.25$ and $0.5$ indicate that frequency-domain diffusion is more suitable for refining local dynamics that are not fully captured in the time domain, rather than serving as the primary noise carrier. For diffusion schedules, the optimal temporal noise level appears at $\beta_\mathrm{end}^t=0.5$, while either excessive ($0.75$) or insufficient ($0.25$) noise weakens the model's ability to learn the underlying data distribution. Regarding $\beta_{\mathrm{end}}^f$, small values provide insufficient spectral information for refinement, whereas overly large values over-regularize the temporal branch and lead to degraded performance.


Finally, $\omega=0.4$ achieves the best overall trade-off, yielding the lowest RMSE across all missing rates while maintaining competitive MAE. Smaller weights lead to insufficient coupling, whereas larger weights may slightly improve some MAE values but generally degrade RMSE, suggesting over-constrained cross-domain fitting.

\begin{table*}[tbp]
\centering
\caption{Results of ablation studies and parameter analysis. "FA" stands for the proposed frequency-aware diffusion embedding and "DE" stands for the diffusion embedding in \citep{DBLP:conf/nips/TashiroSSE21}.}
\label{tab:abl}
\resizebox{0.75\linewidth}{!}{
\begin{sc}
\begin{tabular}{lllllllllll}
\hline
\multirow{2}{*}{Embedding} & \multirow{2}{*}{$\lambda$} & \multirow{2}{*}{$\beta_\mathrm{end}^t$} & \multirow{2}{*}{$\beta_\mathrm{end}^f$} & \multirow{2}{*}{$\omega$} & \multicolumn{2}{l}{PhysioNet 10\%} & \multicolumn{2}{l}{PhysioNet 50\%} & \multicolumn{2}{l}{PhysioNet 90\%} \\ \cline{6-11} 
                           &                            &                                         &                                         &                           & RMSE             & MAE             & RMSE             & MAE             & RMSE             & MAE             \\ \hline
FA                         & 0.75                       & 0.5                                     & 0.05                                    & 0.4                       & 0.423            & 0.206           & 0.622            & 0.288           & 0.782            & 0.448           \\ \hline
\textbf{DE}                & 0.75                       & 0.5                                     & 0.05                                    & 0.4                       & 0.439            & 0.208           & 0.638            & 0.291           & 0.798            & 0.451           \\ \hline
FA                         & \textbf{0.25}              & 0.5                                     & 0.05                                    & 0.4                       & 0.471            & 0.235           & 0.661            & 0.320           & 0.840            & 0.489           \\
FA                         & \textbf{0.5}               & 0.5                                     & 0.05                                    & 0.4                       & 0.503            & 0.216           & 0.660            & 0.299           & 0.815            & 0.462           \\ \hline
FA                         & 0.75                       & \textbf{0.25}                           & 0.05                                    & 0.4                       & 0.451            & 0.210           & 0.664            & 0.294           & 0.806            & 0.455           \\
FA                         & 0.75                       & \textbf{0.75}                           & 0.05                                    & 0.4                       & 0.475            & 0.209           & 0.667            & 0.312           & 0.810            & 0.473           \\ \hline
FA                         & 0.75                       & 0.5                                     & \textbf{0.01}                           & 0.4                       & 0.490            & 0.213           & 0.650            & 0.294           & 0.806            & 0.456           \\
FA                         & 0.75                       & 0.5                                     & \textbf{0.1}                            & 0.4                       & 0.433            & 0.210           & 0.639            & 0.290           & 0.797            & 0.451           \\
FA                         & 0.75                       & 0.5                                     & \textbf{0.25}                           & 0.4                       & 0.464            & 0.225           & 0.661            & 0.306           & 0.812            & 0.467           \\
FA                         & 0.75                       & 0.5                                     & \textbf{0.5}                            & 0.4                       & 0.581            & 0.261           & 0.700            & 0.341           & 0.819            & 0.475           \\ \hline
FA                         & 0.75                       & 0.5                                     & 0.05                                    & \textbf{0.0}              & 0.432            & 0.209           & 0.629            & 0.290           & 0.798            & 0.451           \\
FA                         & 0.75                       & 0.5                                     & 0.05                                    & \textbf{0.2}              & 0.441            & 0.208           & 0.649            & 0.290           & 0.800            & 0.45            \\
FA                         & 0.75                       & 0.5                                     & 0.05                                    & \textbf{0.6}              & 0.441            & 0.206           & 0.641            & 0.289           & 0.797            & 0.449           \\
FA                         & 0.75                       & 0.5                                     & 0.05                                    & \textbf{0.8}              & 0.430            & 0.205           & 0.652            & 0.289           & 0.799            & 0.451           \\
FA                         & 0.75                       & 0.5                                     & 0.05                                    & \textbf{1.0}              & 0.441            & 0.205           & 0.651            & 0.288           & 0.794            & 0.447           \\ \hline
\end{tabular}
\end{sc}}
\vspace{-5mm}
\end{table*}

\section{Related Works}
Time series imputation aims to recover complete sequences from partially observed data. Early methods such as mean and linear interpolation (Lerp) rely on simple statistical assumptions and often fail to capture complex temporal dynamics. Deep learning-based approaches, including BRITS~\citep{DBLP:conf/nips/CaoWLZLL18}, SAITS~\citep{DBLP:journals/eswa/DuCL23}, and TimesNet~\citep{DBLP:conf/iclr/WuHLZ0L23}, improve imputation by modeling complex temporal dependencies.

Diffusion models have recently been introduced to time series imputation, motivated by their success in other domains~\citep{DBLP:conf/iccv/PeeblesX23,DBLP:conf/icml/RasulSSV21,DBLP:conf/nips/SahariaCSLWDGLA22}. CSDI~\citep{DBLP:conf/nips/TashiroSSE21} formulates imputation as conditional score matching with observed values as conditions, while SSSD~\citep{DBLP:journals/tmlr/AlcarazS23} incorporates state space models for temporal modeling. Recent works also explore frequency-domain information: FGTI~\citep{DBLP:conf/nips/YangSYC24} introduces frequency-space conditions for multi-domain representation learning, and FourierDiffusion~\citep{DBLP:conf/icml/CrabbeHSS24} transforms signals into the spectral domain and conducts diffusion entirely in frequency space.

Frequency-domain modeling has also been widely studied in time-series analysis, as spectral characteristics provide complementary information beyond the time domain. Fedformer~\citep{DBLP:conf/icml/ZhouMWW0022} combines decomposition with a frequency-domain Transformer and frequency mode selection for efficient long-term forecasting. FiLM~\citep{DBLP:conf/nips/ZhouMWW0YY022} uses Legendre memory with Fourier- and low-rank-based frequency-enhanced layers to preserve dominant low-frequency patterns and suppress noise. TFAD~\citep{DBLP:conf/cikm/ZhangZWS22} combines time- and frequency-domain analysis with sequence decomposition, context comparison, and data augmentation for anomaly detection, where the frequency branch mainly provides a discriminative representation space.

HyFAD differs from prior works by explicitly coupling time- and frequency-domain diffusion for time-series imputation. Rather than using frequency information only for representation learning, HyFAD performs forward and reverse diffusion directly in both domains. It further introduces frequency-aware diffusion embedding and step-aware spectral guidance to support progressive denoising and reconstruction, improving recovery of both global trends and fine-grained temporal variations.

\section{Conclusion}

In this paper, we propose HyFAD, a hybrid time-frequency diffusion model with frequency-aware diffusion embedding. By integrating time- and frequency-domain processes within a unified diffusion framework, HyFAD enables coarse-to-fine reconstruction. The proposed frequency-aware diffusion-step embedding adaptively emphasizes appropriate spectral components during denoising, enhancing reconstruction of high-frequency components. Experiments demonstrate consistently strong performance across multiple datasets and varying missing rates.

\clearpage
{
\small

\bibliography{example_paper}
\bibliographystyle{unsrtnat}
}

\appendix
\section{Appendix}
\label{sec:appendix}
\subsection{Choice of $\mathcal{F}$ and $\Lambda$}
Notably, our framework imposes no strict restrictions on the transform $\mathcal{F}$, as any invertible linear frequency-domain transformation satisfies our requirements. On the other hand, the $\Lambda$ matrix is intrinsically coupled with the chosen transform $\mathcal{F}$. Any modification to the transform $\mathcal{F}$ necessitates a corresponding adjustment to the $\Lambda$ matrix to ensure energy conservation between the time and frequency domains.
\subsection{Details in Forward Process}
\textbf{Posterior of Eq.\ref{eq:addnoisef2}}: Eq.\ref{eq:addnoisef2} presents 
\begin{equation}
\mathcal{F}^{-1}(\mathbf{x}_k^f) = \sqrt{\alpha_k^f}\mathbf{x}_{k-1}^t + \sqrt{\beta_k^f}\sqrt{1-\lambda}\mathcal{F}^{-1}(\Lambda\boldsymbol{\epsilon}_k^f),\qquad\boldsymbol{\epsilon}_k^f\sim\mathcal{N}(0,\mathbf{I}). 
\end{equation}

$\mathcal{F}^{-1}(\mathbf{x}_{k-1}^t)$ is a linear transformation of $\mathbf{x}_k^t$ along with a Gaussian term, Therefore, the posterior $q(\mathcal{F}^{-1}(\mathbf{x}_{k}^f)\vert \mathbf{x}_{k-1}^t)$ remains gaussian. The posterior mean is $\sqrt{\alpha_k^f}\mathbf{x}_{k-1}^t$ and the standard deviation is 
\begin{equation}
(\sqrt{\beta_k^f}\sqrt{1-\lambda}\mathcal{F}^{-1}\Lambda)(\sqrt{\beta_k^f}\sqrt{1-\lambda}\mathcal{F}^{-1}\Lambda)^T=\beta_k^f(1-\lambda)\mathcal{F}^{-1}\Lambda\Lambda^T(\mathcal{F}^{-1})^T
\end{equation}
Here $\mathcal{F}^{-1}$ and $(\mathcal{F}^{-1})^T$ denotes a matrix $\mathbf{U}$ due to the linearity of transformation $\mathcal{F}$, \ie{$\mathcal{F}(\mathbf{x}) = \mathbf{Ux}$}.

\textbf{Posterior of Eq.\ref{eq:addnoiset}}: Eq.\ref{eq:addnoiset} presents $    \mathbf{x}_k^t = \sqrt{\alpha_k^t}\mathcal{F}^{-1}(\mathbf{x}_k^f) + \sqrt{\beta_k^t}\sqrt{\lambda}\boldsymbol{\epsilon}_k^t$, $\boldsymbol{\epsilon}_k^t\sim\mathcal{N}(0,\mathbf{I})$.$\mathbf{x}_k^t$ is a linear transformation of $\mathcal{F}^{-1}(\mathbf{x}_k^f)$ with a Gaussian term $\sqrt{\beta_k^t}\sqrt{\lambda}\boldsymbol{\epsilon}_k^t$. Therefore, the posterior is still Gaussian with mean $\sqrt{\alpha_{k}^t}\mathcal{F}^{-1}(\mathbf{x}_k^f)$ and standard deviation $\beta_k^t\lambda\mathbf{I}$.

\textbf{Proof of Eq.\ref{eq:addnoisetotal}}~(the relationship between $\mathbf{x}_k^t$ and $\mathbf{x}_k^0$ in the forward process):

From Eq.\ref{eq:addnoiseall}:
\begin{equation}
    \mathbf{x}_k^t = \sqrt{\alpha_k^t\alpha_k^f}\mathbf{x}_{k-1}^t + \sqrt{\alpha_k^t(1-\alpha_k^f)}\sqrt{1-\lambda}\mathcal{F}^{-1}(\Lambda\boldsymbol{\epsilon}_k^f)+\sqrt{1-\alpha_k^t}\sqrt{\lambda}\boldsymbol{\epsilon}_k^t
    \label{eq:appen1}
\end{equation}
we can get:
\begin{equation}
    \mathbf{x}_k^t = \sqrt{\bar{\alpha}_k^t\bar{\alpha}_k^f}\mathbf{x}_0^t + \sqrt{1-\lambda}\sum_{s=1}^k\sqrt{\beta_{s}^f}\sqrt{\frac{\bar{\alpha}_k^t}{\bar{\alpha}_{s-1}^t}}\sqrt{\frac{\bar{\alpha}_k^f}{\bar{\alpha}_s^f}}\mathcal{F}^{-1}(\Lambda\boldsymbol{\epsilon}_s^f) + \sqrt{\lambda}\sum_{s=1}^k\sqrt{\beta_s^t}\sqrt{\frac{\bar{\alpha}_k^t}{\bar{\alpha}_s^t}}\boldsymbol{\epsilon}_s^t
    \label{eq:appen2}
\end{equation}
\begin{proof}
We proceed by mathematical induction.  

For $k = 1$:
\begin{equation}
\begin{aligned}
    \mathbf{x}_1^t &= \sqrt{\alpha_1^t\alpha_1^f}\mathbf{x}_0^t + \sqrt{\alpha_1^t(1-\alpha_1^f)}\sqrt{1-\lambda}\mathcal{F}^{-1}(\Lambda\boldsymbol{\epsilon}_1^f)+\sqrt{1-\alpha_1^t}\sqrt{\lambda}\boldsymbol{\epsilon}_1^t \\
    &= \sqrt{\bar{\alpha}_1^t\bar{\alpha}_1^f}\mathbf{x}_0^t + \sqrt{1-\lambda}\sqrt{\beta_1^f}\sqrt{\alpha_1^t}\mathcal{F}^{-1}(\Lambda\boldsymbol{\epsilon}_1^f) + \sqrt{\lambda}\sqrt{\beta_1^t}\boldsymbol{\epsilon}_1^t\\
    &= \sqrt{\bar{\alpha}_1^t\bar{\alpha}_1^f}\mathbf{x}_0^t + \sqrt{1-\lambda}\sum_{s=1}^1\sqrt{\beta_1^f}\sqrt{\frac{\bar{\alpha}_1^t}{\bar{\alpha}_0^t}}\sqrt{\frac{\bar{\alpha}_1^f}{\bar{\alpha}_1^f}}\mathcal{F}^{-1}(\Lambda\boldsymbol{\epsilon}_1^t) + \sqrt{\lambda}\sum_{s=1}^1\sqrt{\beta_1^t}\sqrt{\frac{\bar{\alpha}_1^t}{\bar{\alpha}_1^t}}\boldsymbol{\epsilon}_1^t
\end{aligned}
\end{equation}
Therefore, Eq.\ref{eq:appen2} holds when $k=1$. Suppose Eq.\ref{eq:appen2} holds when $k=m$, \ie{}
\begin{equation}
    \mathbf{x}_m^t = \sqrt{\bar{\alpha}_m^t\bar{\alpha}_m^f}\mathbf{x}_0^t + \sqrt{1-\lambda}\sum_{s=1}^m\sqrt{\beta_{s}^f}\sqrt{\frac{\bar{\alpha}_m^t}{\bar{\alpha}_{s-1}^t}}\sqrt{\frac{\bar{\alpha}_m^f}{\bar{\alpha}_s^f}}\mathcal{F}^{-1}(\Lambda\boldsymbol{\epsilon}_s^f) + \sqrt{\lambda}\sum_{s=1}^m\sqrt{\beta_s^t}\sqrt{\frac{\bar{\alpha}_m^t}{\bar{\alpha}_s^t}}\boldsymbol{\epsilon}_s^t
\end{equation}
For $k=m+1$:
\begin{equation}
\begin{aligned}
\mathbf{x}_{m+1}^t &= \sqrt{\alpha_{m+1}^t\alpha_{m+1}^f}\mathbf{x}_{m}^t + \sqrt{\alpha_{m+1}^t(1-\alpha_{m+1}^f)}\sqrt{1-\lambda}\mathcal{F}^{-1}(\Lambda\boldsymbol{\epsilon}_{m+1}^f)+\sqrt{1-\alpha_{m+1}^t}\sqrt{\lambda}\boldsymbol{\epsilon}_{m+1}^t    \\
&= \sqrt{\alpha_{m+1}^t\alpha_{m+1}^f}(\sqrt{\bar{\alpha}_m^t\bar{\alpha}_m^f}\mathbf{x}_0^t + \sqrt{1-\lambda}\sum_{s=1}^m\sqrt{\beta_{s}^f}\sqrt{\frac{\bar{\alpha}_m^t}{\bar{\alpha}_{s-1}^t}}\sqrt{\frac{\bar{\alpha}_m^f}{\bar{\alpha}_s^f}}\mathcal{F}^{-1}(\Lambda\boldsymbol{\epsilon}_s^f)  \sqrt{\lambda}\sum_{s=1}^m\sqrt{\beta_s^t}\sqrt{\frac{\bar{\alpha}_m^t}{\bar{\alpha}_s^t}}\boldsymbol{\epsilon}_s^t)\\
&+ \sqrt{\alpha_{m+1}^t(1-\alpha_{m+1}^f)}\sqrt{1-\lambda}\mathcal{F}^{-1}(\Lambda\boldsymbol{\epsilon}_{m+1}^f)+\sqrt{1-\alpha_{m+1}^t}\sqrt{\lambda}\boldsymbol{\epsilon}_{m+1}^t \\
&= \sqrt{\bar{\alpha}_{m+1}^t\bar{\alpha}_{m+1}^f}\mathbf{x}_0^t + \sqrt{1-\lambda}\sum_{s=1}^{m+1}\sqrt{\beta_{s}^f}\sqrt{\frac{\bar{\alpha}_{m+1}^t}{\bar{\alpha}_{s-1}^t}}\sqrt{\frac{\bar{\alpha}_{m+1}^f}{\bar{\alpha}_s^f}}\mathcal{F}^{-1}(\Lambda\boldsymbol{\epsilon}_s^f) + \sqrt{\lambda}\sum_{s=1}^{m+1}\sqrt{\beta_s^t}\sqrt{\frac{\bar{\alpha}_{m+1}^t}{\bar{\alpha}_s^t}}\boldsymbol{\epsilon}_s^t
\end{aligned}
\end{equation}
Therefore, Eq.\ref{eq:appen2} holds for arbitrary $k$.
\end{proof}
\textbf{Posterior of Eq.\ref{eq:addnoisetotal}}: Eq.\ref{eq:addnoisetotal} indicates $\mathbf{x}_k^t$ is linear transformation of $\mathbf{x}_0^t$ with two groups of independent gaussian noise. Therefore, $q(\mathbf{x}_k^t\vert\mathbf{x}_0^t)$ is still gaussian with the standard deviation of the sum of two the two groups of gaussian noise. 
\subsection{Details in the reverse process}
\textbf{Noise prior in the reverse process.} At the end of the forward process, $\bar{\alpha}_k^t,\bar{\alpha}_k^f\to 0$, therefore, the mean of $\mathbf{x}_k^t$ is $0$. For the standard deviation term, it is the linear combination of two independent gaussian noises, so the reverse process starts from $\mathbf{x}_{T}=\sqrt{\lambda}\boldsymbol{\epsilon}^t+\sqrt{1-\lambda}\mathcal{F}^{-1}(\Lambda\boldsymbol{\epsilon}^f),\boldsymbol{\epsilon}^t,\boldsymbol{\epsilon}^f\sim\mathcal{N}(0,\mathbf{I})$.

\subsection{Proof of Proposition \ref{prop:freq}}
\label{proof:freq}
\begin{proposition}[Frequency Components in the Diffusion Process]
\label{prop:freqa}
Consider the forward process where the noise is isotropic and has cumulative spectral energy $N_t$ at diffusion step $t$. Partition the frequency axis into ordered bands $\{B_b\}_{b=1}^B$ from low to high frequency, and let $E_b$ denote the band-averaged clean-signal energy in band $B_b$. If the signal satisfies the relative low-pass condition $E_{b_1} \ge E_{b_2}$ for any $b_1<b_2$, then for any SNR threshold $\gamma>0$, the threshold time
$
\tau_b(\gamma)
= \inf \left\{t \in [0,T]: \frac{E_b}{N_t} \le \gamma \right\}
$
satisfies
$
\tau_{b_2}(\gamma) \le \tau_{b_1}(\gamma),
\forall b_1<b_2.
$
That is, higher-frequency bands reach the same degradation threshold no later than lower-frequency bands during the forward process. Equivalently, the reverse process tends to recover low-frequency structures before progressively restoring high-frequency details.
\end{proposition}
\begin{proof}
Considering a diffusion process 
\begin{equation}
\mathrm{d}\mathbf{x}_t = \bm{f(x,t)}\mathrm{d}t + g(t)\mathrm{d}\bm{w_t}
\label{eq:diff}
\end{equation}
where $\bm{w_t}$ is the standard Wiener process in $\mathbb{R}^{\mathrm{d}_{\bm{x}}}$, 
$\bm{f}:\mathbb{R}^{\mathrm{d}_{\bm{x_t}}}\times [0,T]\to \mathbb{R}^{\mathrm{d}_{\bm{x_t}}}$ is the drift 
and $g(t):[0,T]\to\mathbb{R}^{N\times N}$ is the diffusion coefficient. The solution to 
Eq.\ref{eq:diff} is formulated as:
\begin{equation}
    \mathbf{x}_t = \mathbf{x}_0 + \int_0^t\bm{f}(\mathbf{x}_s,s)\mathrm{d}s + \int_0^tg(s)\mathrm{d}\bm{w_s}
    \label{eq:rev}
\end{equation}
By applying DFT to Eq.\ref{eq:rev}, we have: \begin{equation}
    \hat{\mathbf{x}}_t(\omega) = \hat{\mathbf{x}}_0(\boldsymbol{\omega}) + 
    \hat{f}(\boldsymbol{\omega}) + \hat{\boldsymbol{\epsilon}}_t(\boldsymbol{\omega}),
    \label{eq:xtfreq}
\end{equation}
where $\hat{\mathbf{x}}_t(\omega)$ is the DFT of the noisy item and satisfies:
$       \mathbb{E}[ \hat{\boldsymbol{\epsilon}}_t(\omega)] = 0 $ and $
        \mathbb{E}[\vert  \hat{\boldsymbol{\epsilon}}_t(\omega)\vert^2] = \int_0^t\vert g(s)\vert^2\mathrm{d} s$.
The signal-noise-ratio~(SNR) of frequency $\omega$ is defined as:
\begin{equation}
    \text{SNR}(\omega) = \frac{\vert\hat{\mathbf{x}}_t(\omega)\vert^2}{\mathbb{E}[\vert  \hat{\boldsymbol{\epsilon}}_t(\omega)\vert^2]}
    = \frac{\vert\hat{\mathbf{x}}_t(\omega)\vert^2}{\int_0^t\vert g(s)\vert^2\mathrm{d} s}
    \label{eq:SNR}
\end{equation}
We partition the frequency axis into $B$ disjoint bands $\{\mathcal{B}_b\}_{b=1}^B$ from low to high frequency. Let $N_t := \int_0^t |g(s)|^2\,\mathrm{d}s$ denote the cumulative noise intensity up to diffusion step $t$. We define the band-averaged signal energy as $E_b := \mathbb{E}_{\omega \in \mathcal{B}_b}\!\left[ \mathbb{E}\big[|\hat{\mathbf{x}}_0(\omega)|^2\big] \right]$. Under the assumption that the second-order statistics of the transformed noise do not explicitly depend on frequency, the band-level SNR for any band $b$ can be written as
$$
\mathrm{SNR}_b(t)=\frac{E_b}{N_t}.
$$
For any $\gamma>0$, we have
$$
\mathrm{SNR}_b(t)\le\gamma
\iff
\frac{E_b}{N_t}\le\gamma
\iff
N_t\ge \frac{E_b}{\gamma}.
$$
For any $b_1<b_2$, the relative low-pass assumption gives
$$
E_{b_1}\ge E_{b_2}
\quad\Longrightarrow\quad
\frac{E_{b_1}}{\gamma}\ge \frac{E_{b_2}}{\gamma}.
$$
Therefore, if at some time $t$,
$$
N_t\ge \frac{E_{b_1}}{\gamma},
$$
then it must also hold that
$$
N_t\ge \frac{E_{b_2}}{\gamma}.
$$
Equivalently, if the lower-frequency band $b_1$ already satisfies
$$
\mathrm{SNR}_{b_1}(t)\le\gamma,
$$
then the higher-frequency band $b_2$ must also satisfy
$$
\mathrm{SNR}_{b_2}(t)\le\gamma.
$$
This shows that, in the forward diffusion process, higher-frequency bands reach the same SNR threshold no later than lower-frequency bands. Furthermore, since $N_t$ is non-decreasing with respect to $t$, the first threshold-crossing times satisfy
$$
\tau_{b_2}(\gamma)\le \tau_{b_1}(\gamma).
$$
\end{proof}
\subsection{Details of Frequency-aware Diffusion Embedding}
\label{sec:appnafe}

\textbf{Factor $s$ in Noise-aware Frequency Gate}.~Due to substantial variations in spectral energy scales across different datasets and variables, we introduce a global calibration factor $s$ to align the magnitude of band-wise spectral energy, which ensures the thresholding operation in the frequency-aware diffusion embedding is performed on a consistent relative scale.

During training, the factor $s$ is updated per batch using a global exponential moving average~(EMA). Specifically, given the band-wise spectral energy $P_b^s$ of the current batch, $s_0$ is initialized as the median of $P_b^s$:
\begin{equation}
    s_0 = \mathrm{median}(P_b^s),
\end{equation}
$s$ is then updated as:
\begin{equation}
    s_k \leftarrow \alpha s_{k-1} + (1-\alpha)\hat{s}_{k},
\end{equation}
where $\alpha \in (0,1)$ denotes the EMA decay rate and we set $\alpha=0.99$ in our implementation, $s_{k-1}$ is the factor after batch $k-1$ and $\hat{s}_k$ is the observations from batch $k$, \ie{$\hat{s}_k=\mathrm{median}(P_k^s)$}. During inference, the factor $s$ is fixed to the value accumulated during training and is no longer updated. This design enforces a strict separation between training and testing statistics, thereby preventing potential data leakage and ensuring consistent gating behavior across test samples.

\textbf{$\omega_b$ and $c(t)$ in Stage-aware Frequency Schedule}.~Let $b\in\{0,1,\cdots,B-1\}$ index the $B=\lfloor \frac{L}{2}\rfloor + 1$ frequency bands. The relative position of the $b$-th band is:
\begin{equation}
    \tilde{\omega}_b = \frac{b}{B-1},
\end{equation}
To align the relative frequency coordinate with the numerical scale of the stage-dependent cutoff, we further multiply a max frequency $\omega_{\max}$:
\begin{equation}
    \omega_b = \tilde{\omega}_b\cdot \omega_{\max}
\end{equation}

The function $c(t)$ is a stage-dependent frequency scaling function that controls the effective frequency range emphasized by the model at diffusion step $t$, which is defined as:
\begin{equation}
    c(t) = c_{\min} + (c_{\max} - c_{\min})\left(1-\frac{t}{T}\right)^q,
    \label{eq:ct}
\end{equation}
where $t$ is the current diffusion step and $T$ is the total number of diffusion steps. $c_{\min}$ and $c_{\max}=\omega_{\max}$ are the minimum and maximum cutoff frequency, $q>0$ is a transition exponent. $c(t)$ remains close to $c_{\min}$ during the early stages of the reverse diffusion process, thereby emphasizing low-frequency components. As the denoising process progresses, $c(t)$ gradually approaches $c_{\max}$, allowing the model to progressively incorporate higher-frequency details.

\textbf{Embedding Details}. The projection operator $\mathrm{Proj}(\cdot,\tfrac{d}{2})$ in Eq.\ref{eq:proj} is implemented via one-dimensional linear interpolation along the frequency axis
using \texttt{torch.nn.functional.interpolate} (\texttt{mode='linear'}),
which resamples the frequency-wise gating vector from
$\lfloor L/2\rfloor+1$ frequency bins to $\tfrac{d}{2}$ embedding dimensions. 

The frequency grid $\mathbf{f}_d= [f_1,f_2,\cdots,f_{\frac{d}{2}}]\in \mathbb{R}^{\frac{d}{2}}$ in Eq.\ref{eq:grid} is constructed by uniform sampling from $[0,f_{\max}]$:
\begin{equation}
    f_i = \frac{i-1}{\frac{d}{2}-1}f_{\max}
\end{equation}

\subsection{Architecture Details}
\label{sec:archdetails}
\begin{figure}[htbp]
    \centering
    \includegraphics[width=0.95\linewidth]{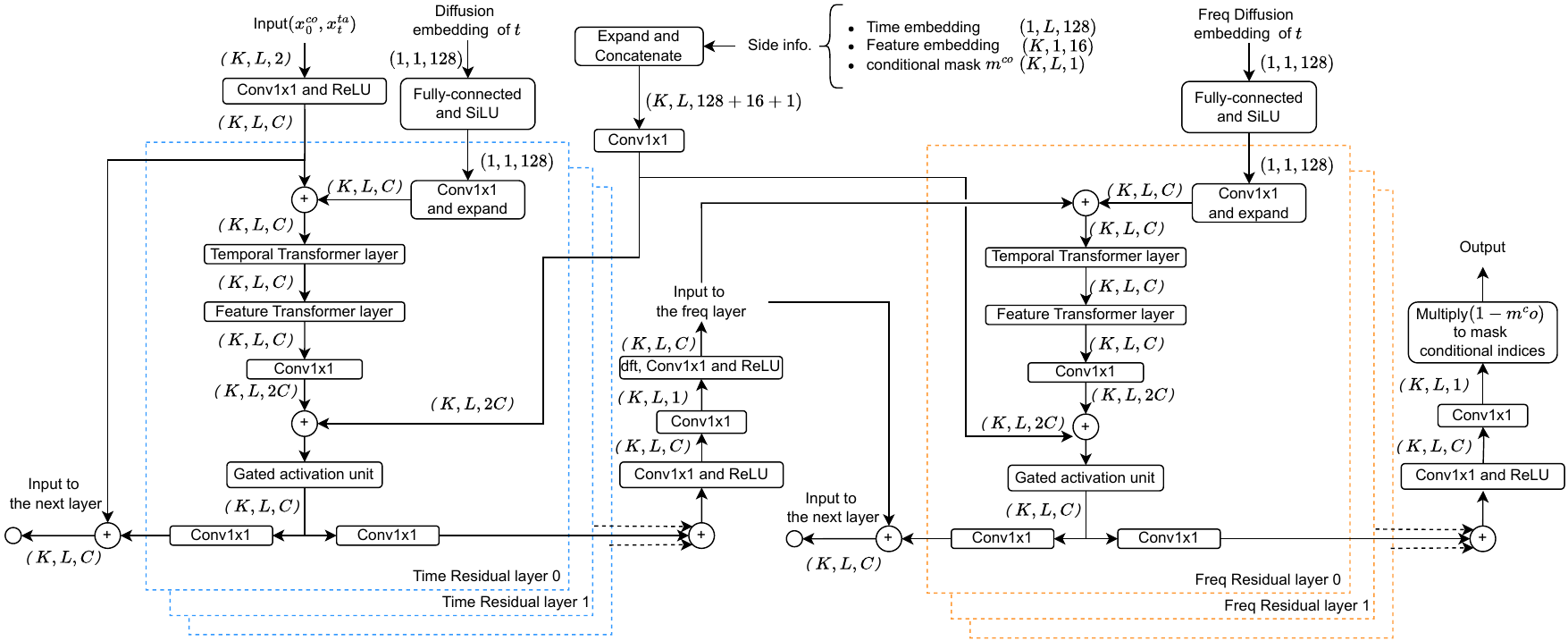}
    \caption{Architecture details of our denoising model $\boldsymbol{\epsilon}_{\theta}$.}
    \label{fig:arch}
\end{figure}
Fig.\ref{fig:arch} presents the diffusion backbone of our denoising block. The block adopts a time-frequency coupled dual-branch residual The input sequence is first projected using a $1\times1$ convolution followed by a ReLU activation. Time and Frequency diffusion step embeddings are generated via fully connected layers with SiLU activations and then concatenated with external side information.

During residual modeling, the network first processes the representation in the time-domain branch. After an initial $1\times1$ convolution and ReLU projection, the features are passed through a temporal Transformer layer and a feature Transformer layer to capture inter-channel and intra-channel dependencies. The outputs are fused using a gated activation unit (GAU) and further refined by stacked $1\times1$ convolutions. Residual connections feed the results back into the input stream to form the input for the next layer.

Then the time-domain representation is then fed into the frequency-domain branch with a similar structure to get the frequency domain output,. By processing in a time-to-frequency order, the architecture establishes a complementary relationship between global trend modeling and fine-grained frequency refinement.
\subsection{Training and Sampling Details}
\label{sec:alg}
The detailed training and sampling algorithm is presented in Alg.\ref{alg:training} and \ref{alg:inference}.
\begin{algorithm}[htbp]
  \caption{Training Procedure of HyFAD}
  \label{alg:training}
  \begin{algorithmic}[1]
    \STATE \textbf{Input:} Observed sequence $\mathbf{x}_0$, condition mask $\mathbf{M}^{\text{cond}}$, observation mask $\mathbf{M}^{\text{obs}}$,
    side information $\mathbf{s}$, number of steps $T$,number of iterations $N$, time- and frequency-domain scheduler $\beta^t,\beta^f$, total denoising step $T$, transformation $\mathcal{F}$, noise balance coefficient $\lambda$, consistency weight $\omega$.
    \STATE \textbf{Output:} Denoising function $\boldsymbol{\epsilon}_{\theta}^t, \boldsymbol{\epsilon}_{\theta}^f$
    \FOR{$i=1$ \textbf{to} T}
    \STATE $k\sim$ \text{Uniform}$(\{1,2,\cdots,T\})$
    \STATE Calculate time-domain noise:$\boldsymbol{\epsilon}_k^t = \sqrt{\lambda}\sum_{s=1}^k\sqrt{\beta_s^t}\sqrt{\frac{\bar{\alpha}_k^t}{\bar{\alpha}_s^t}}\boldsymbol{\epsilon}_s^t,$
    $\boldsymbol{\epsilon}_s^t\sim\mathcal{N}(0,\mathbf{I})$.
    \STATE Calculate frequency-domain noise:$\boldsymbol{\epsilon}_k^f=\sqrt{1-\lambda}\sum_{s=1}^k\sqrt{\beta_{s}^f}\sqrt{\frac{\bar{\alpha}_k^t}{\bar{\alpha}_{s-1}^t}}\sqrt{\frac{\bar{\alpha}_k^f}{\bar{\alpha}_s^f}}\mathcal{F}^{-1}(\Lambda\boldsymbol{\epsilon}_s^f),\boldsymbol{\epsilon}_s^f\sim\mathcal{N}(0,\mathbf{I})$. 
    \STATE Calculate noisy sample at step $k$:$\mathbf{x}_k^t = \sqrt{\alpha_k^t\alpha_k^f}\mathbf{x}_0^t+\boldsymbol{\epsilon}_k^t+\boldsymbol{\epsilon}_k^f$
    \STATE Construct time domain input: $\mathbf{h}^t \leftarrow \texttt{set\_input}(\mathbf{x}_k^t,\mathbf{X}^{\text{obs}},\mathbf{M}^{\text{cond}})$
    \STATE Estimate time domain noise and calculate time noise loss:$\hat{\boldsymbol{\epsilon}}_k^t = \boldsymbol{\epsilon}_{\theta}^t(\mathbf{h}_t,\mathbf{s},k)$, $\mathcal{L}^t=\Vert\hat{\boldsymbol{\epsilon}}_k^t-\boldsymbol{\epsilon}_k^t\Vert_2^2$
    \STATE Update and construct frequency domain input:$\mathcal{F}^{-1}(\mathbf{x}_k^f)= \frac{1}{\sqrt{\alpha_k^t}}\left(\mathbf{x}_k^t - \frac{\beta_k^t}{\sqrt{1-\bar{\alpha}_k^t}}\hat{\boldsymbol{\epsilon}}_k^t\right)$
    \STATE Construct frequency domain input:$\mathbf{h}^f = \texttt{set\_input}(\mathcal{F}^{-1}(\mathbf{x}_k^t),\mathbf{s},k)$
    \STATE Estimate frequency domain noise and calculate frequency noise loss: $\hat{\boldsymbol{\epsilon}}_k^f=\hat{\boldsymbol{\epsilon}}_{\theta}(\mathbf{h}^f,\mathbf{s},k)$, $\mathcal{L}^f = \Vert\mathcal{F}^{-1}(\Lambda\hat{\boldsymbol{\epsilon}}_k^f)-\boldsymbol{\epsilon}_k^f\Vert_2^2$
    \STATE Calculate consistency loss:$\mathcal{L}_{c} = \Vert(\boldsymbol{\epsilon}_k^t+\boldsymbol{\epsilon}_k^f)-(\hat{\boldsymbol{\epsilon}}_k^t+\mathcal{F}^{-1}(\Lambda\hat{\boldsymbol{\epsilon}}_k^f))\Vert_2^2$
    \STATE Take gradient descent step on $\nabla_{\theta}\left[(1-\mathbf{M}^{\text{cond}})\cdot(\mathcal{L}^t+\mathcal{L}^f+\omega\mathcal{L}_c)\right]$
    \ENDFOR
    \RETURN {$\boldsymbol{\epsilon}_{\theta}^t,\boldsymbol{\epsilon}_{\theta}^f$}
  \end{algorithmic}
\end{algorithm}
\begin{algorithm}[htbp]
  \caption{Sampling Procedure of HyFAD}
  \label{alg:inference} 
  \begin{algorithmic}[1]
    \STATE\textbf{Input:} Trained denoising function $\boldsymbol{\epsilon}_{\theta}^t,\boldsymbol{\epsilon}_{\theta}^f$, sampling step $T$, observed data $\mathbf{X}^{\text{obs}}$, condition mask $\mathbf{M}^{\text{cond}}$, side information $\mathbf{s}$,
    time domain scheduler $\beta^t$, frequency domain scheduler $\beta^f$
    \STATE\textbf{Output:} Generated Sequence $\mathbf{x}_0$
    \STATE Initialize by sampling from noise prior: $\mathbf{x}_T=\sqrt{\lambda}\mathbf{z}^t+\sqrt{1-\lambda}\mathcal{F}^{-1}(\Lambda\mathbf{z}^f).\quad \mathbf{z}^t,\mathbf{z}^f\sim\mathcal{N}(0,\mathbf{I})$
    \FOR{$k=T$ \textbf{to} $1$}
    \STATE Construct time-domain input:$\mathbf{h}_k^t \leftarrow \texttt{set\_input}(\mathbf{x}_k^t, \mathbf{X}^{\text{obs}}, \mathbf{M}^{\text{cond}})$
    \STATE Estimate time-domain noise:$\hat{\boldsymbol{\epsilon}}_k^t \leftarrow \boldsymbol{\epsilon}_\theta^t(\mathbf{h}_k^t, \mathbf{s}, k)$
    \STATE Update input according to time-domain noise:$\mathcal{F}^{-1}(\mathbf{x}_k^f) = \frac{1}{\sqrt{\alpha_k^t}}\left(\mathbf{x}_k - \frac{\beta_k^t}{\sqrt{1-\bar{\alpha}_k^t}}\hat{\boldsymbol{\epsilon}}_k^t\right)$
    \STATE Construct frequency-domain input:$\mathbf{h}_k^f \leftarrow \texttt{set\_input}(\mathcal{F}^{-1}(\mathbf{x}_k^f), \mathbf{X}^{\text{obs}}, \mathbf{M}^{\text{cond}})$
    \STATE Estimate frequency-domain noise:$\hat{\boldsymbol{\epsilon}}_k^f \leftarrow \boldsymbol{\epsilon}_\theta^f(\mathbf{h}_k^f, \mathbf{s}, k) $
    \STATE Update input according to frequency domain noise: 
    $\mathbf{x}_{k-1}^t\leftarrow\frac{1}{\sqrt{\alpha_k^f}}\left(\mathcal{F}^{-1}(\mathbf{x}_k^f)-\sqrt{1-\lambda}\mathcal{F}^{-1}(\Lambda\hat{\boldsymbol{\epsilon}}_k^f)\right)$
    \ENDFOR
    \RETURN $\mathbf{x}_0$
  \end{algorithmic}
\end{algorithm}
\subsection{Dataset Details}
\label{sec:dadetails}
The PhysioNet dataset~\citep{silva2012predicting} consists of 4000 irregularly-sampled medical time series data including 35 variables (\eg{Albumin and heart rate}) collected from ICU with a total length of 48 hours and has been made available by their authors under the terms of the Creative Commons Attribution License 3.0 (CCAL). Consistent with previous studies~\citep{DBLP:conf/nips/TashiroSSE21}, the dataset is processed hourly to get 48 timesteps. The processed dataset contains nearly 80\% originally missing values without ground truth. In our experiments, we random select 10/50/90\% of the observed values as the imputation targets~(\ie{ground truth of test dataset}).

The air quality dataset~\citep{zhang2017cautionary} contains PM2.5 data collected from 36 monitor stations in Beijing and is licensed under a Creative Commons Attribution 4.0 International (CC BY 4.0) license.  All the air quality data are collected hourly for 12 months. The original dataset contains 13.3\% missing values with artificial ground-truth with a non-random missing pattern.

The details of the datasets are presented in Tab.\ref{tab:datasetdetails}.
\begin{table}[htbp]
\centering
\caption{Details of PhysioNet and Air Quality datasets.}
\label{tab:datasetdetails}
\resizebox{0.95\linewidth}{!}{
\begin{tabular}{lccccc}
\toprule
\textbf{Dataset}  & \textbf{\# Samples} & \textbf{\# Features} & \textbf{Time Steps} & \textbf{Missing Type}  & \textbf{Missing Ratio}\\ \midrule
PhysioNet         & 4000                & 35                   & 48                  & Originally Missing \& Random & 80\%~(Original)     \\
Air Quality & 5633                & 36                   & 36                  & Non-Random \& Artificial     & 13\%~(Original)     \\ 
\bottomrule
\end{tabular}}
\end{table}
\subsection{Evaluation Metrics}
In this section, we present the details of evaluation metrics in our experiments. $y, \hat{y}\in\mathbb{R}^{K\times L}$ denote the ground truth and output of our model and $M$ is the indicator matrix.

\textbf{Mean Absolute Error~(MAE)} calculates the average $L_1$ error between the imputed samples and the ground truth of the time series:
\begin{equation}
    \mathbf{MAE}(y,\hat{y}) = \frac{\sum_{i=1}^K\sum_{j=1}^LM_{ij}\vert y_{ij}-\hat{y}_{ij}\vert_1}{\sum_{i=1}^K\sum_{j=1}^LM_{ij}}
\end{equation}
MAE reflects the overall deviation across all points, emphasizes the overall accuracy of the model outputs. Compared with RMSE, MAE is more robust to outliers in the data.

\textbf{Rooted Mean Square Error~(RMSE)} calculates the average $L_2$ error between the imputed samples and the ground truth of time series:
\begin{equation}
    \mathbf{RMSE}(y,\hat{y}) = \sqrt{\frac{\sum_{i=1}^K\sum_{j=1}^LM_{ij}\Vert y_{ij}-\hat{y}_{ij}\Vert_2^2}{\sum_{i=1}^K\sum_{j=1}^LM_{ij}}}
\end{equation}
RMSE highlights potential large errors, so a few big deviations may dominate the score. Compared with MAE, RMSE is less robust to outliers but better captures worst-case performance.
\subsection{Limitations}
\label{sec:limitations}
Here we provide an in-depth discussion regarding the limitations of our current work. First, our exploration of frequency-domain representations is primarily confined to the Discrete Fourier Transform~(DFT). Future research could extend this framework to alternative transforms, such as the Discrete Cosine Transform~(DCT) or Discrete Wavelet Transform~(DWT), to better capture different kinds of features from different types of time series data. Second, this study focuses on deterministic time series imputation using DDIM for efficient sampling. A promising future direction involves adapting DDPM-based sampling schemes to our hybrid time-frequency architecture, as well as evaluating the model's performance on probabilistic imputation tasks to better quantify uncertainty. Another limitation is that the frequency-domain sampling process can be less numerically stable than sampling in the time domain, particularly when reconstructing signals from sparse spectral representations. We intend to investigate specialized stabilization methods for frequency-domain diffusion in future works.
\subsection{Implementation Details}
Global factor $s_k$ is updated as $s_k = \alpha s_{k-1} + (1-\alpha)\hat{s}_k$, where $\hat{s}_k$ is computed as the median spectral energy of the current training batch. It depends only on the observed values of the training samples and their diffusion states without any validation or test information. Moreover, this EMA is updated only during training and kept fixed at inference time, so it does not introduce data leakage. 

The noise scheduler does not rely on validation data. The scheduler is predefined, and the noise scale $N_t$ in the frequency branch (Eq. 18) is computed analytically from the diffusion coefficients rather than estimated from validation or test statistics. Therefore, neither the construction nor the use of the scheduler involves validation information.

All baselines are evaluated under the same setting. We strictly follow the CSDI setting for data preprocessing, splitting, and evaluation on both PhysioNet and Air Quality, ensuring a fair comparison across methods. For validation, we follow the same setting as CSDI: the validation set is only used to select the best checkpoint, and we do not use early stopping in our training.
\subsection{Computational Cost}
\label{sec:time}
Compared with CSDI, the extra cost of HyFAD mainly comes from frequency-domain branch, time–frequency transforms and the joint optimization objective. Since the backbone is Transformer-based, the overall complexity is $O(L^2)$ with transform cost $O(L\log L)$. We also provide training and inference speed on PhysioNet (batch size = 16) in Table.\ref{tab:trainingtime}.
\begin{table}[tbp]
\centering
\caption{Training time comparison of CSDI and HyFAD.}
\begin{sc}
\label{tab:trainingtime}
\begin{tabular}{c|cc}
\hline
                 & CSDI   & HyFAD  \\ \hline
training/epoch   & 6.70s  & 18.06s \\
inference/sample & 36.51s & 46.42s \\ \hline
\end{tabular}
\end{sc}
\end{table}
\subsection{Experiment details}
\label{sec:appendetails}
\subsubsection{Hyperparameters}
The hyperparameters in our experiments is detailed in Table.\ref{tab:hyperps}
and the hyperparameters in frequency-aware diffusion embedding in presented in  Table.\ref{tab:hyperfade}
\begin{table}[htbp]
\centering
\caption{Hyperparameter details in our experiments}
\label{tab:hyperps}
\begin{tabular}{@{}lll@{}}
\toprule
Dataset                      & PhysioNet  &  Air Quality        \\ \midrule
Epochs                       & 400      & 400      \\
Batch Size                   & 16       & 16       \\                  Learning Rate                & 0.001    & 0.001 \\
\#Time Layers               & 4        & 4        \\
\#Freq Layers                     & 4        & 4        \\
Channels                     & 128      & 128      \\
\#Time Heads                 & 8        & 8        \\
\#Freq Heads                  & 8        & 8        \\
\#Time Diffusion Embedding Dim & 128      & 128      \\
\#Freq Diffusion Embedding Dim & 128      & 128      \\
$\beta_\mathrm{start}^t$                   & 0.0001 & 0.0001 \\
$\beta_\mathrm{end}^t$                  & 0.5      & 0.75      \\
$\beta_\mathrm{start}^f$                 & 0.0001 & 0.0001 \\
$\beta_\mathrm{end}^f$                   & 0.05     & 0.05     \\
$\omega$            & 0.4      & 1.0      \\
$\lambda$           & 0.75     & 0.75    \\
\#Steps                    & 50       & 50       \\
Noise Scheduler                     & quad     & quad     \\
\#Time Embedding Dim                      & 128      & 128      \\
\#Feature Embedding Dim                   & 16       & 16       \\
\bottomrule
\end{tabular}
\end{table}
\begin{table}[htbp]
\centering
\caption{Hyperparameters in frequency-aware diffusion embedding.}
\label{tab:hyperfade}
\begin{tabular}{llllllll}
\hline
Dataset     & $\gamma$ & $\tau$ & $g_{\min}$ & $\kappa$ & $c_{\min}$ & $c_{\max}$ & $q$ \\ \hline
PhysioNet   & 1.0      & 0.7    & 0.3      & 0.5      & 1.0      & 100.0    & 1.0 \\
Air Quality & 1.0      & 0.7    & 0.3      & 0.5      & 1.0      & 100.0    & 1.0 \\ \hline
\end{tabular}
\end{table}
\subsubsection{Ablation on explicit frequency modeling}
To empirically assess the contribution of explicit frequency modeling, we compare HyFAD with its pure time-domain counterpart. Because the time-domain and frequency-domain diffusion processes in HyFAD are tightly coupled throughout both the forward and reverse diffusion procedures, simply removing the frequency branch from the denoising network would lead to an incomplete and potentially unfair ablation. Instead, we remove the frequency-domain process from both diffusion stages, which reduces the framework to a purely time-domain diffusion model. This counterpart is equivalent to CSDI, an established generative baseline for time-series imputation. Therefore, the comparison between HyFAD and CSDI serves both as a controlled study of the effect of frequency modeling and as a direct comparison with a representative time-domain generative imputation method.

To further examine whether the imputed series preserves periodic structures across different spectral ranges, we adopt Spectral-MAE and partition the spectrum into three sub-intervals and the results are in Tab.\ref{tab:smae-band}. The Spectral-MAE is defined as the normalized PSD estimates of the ground-truth series $P_{\mathrm{GT}}(\omega)$ against those of
the reconstruction $P_{\mathrm{Pred}}(\omega)$:
\begin{equation}
    \mathrm{S}\text{-}\mathrm{MAE} = \frac{1}{\vert\Omega\vert}\sum_{\omega\in\Omega}\lvert\frac{P_{\mathrm{GT}}(\omega)}{\sum_{\omega^\prime}P_{\mathrm{GT}}(\omega^\prime)}-\frac{P_{\mathrm{pred}}(\omega)}{\sum_{\omega^\prime}P_{\mathrm{pred}}(\omega^\prime)}\rvert
\end{equation}
where $\Omega$ is the set of evaluated frequencies. The results show that CSDI exhibits systematic weaknesses in capturing both low- and high-frequency patterns, highlighting the importance of explicit frequency-domain modeling. Since Spectral-MAE can be dominated by low-frequency components, the performance improvement is not significant. We additionally compute the error in the logarithmic spectral domain and report band-wise Log-SMAE. Moreover, we report the Leading Frequency Error (LFE), which evaluates whether the dominant frequency component of the original signal is correctly recovered by the imputed sequence. LFE complements Log-SMAE by focusing on the recovery of the dominant periodic structure: a large LFE indicates that the dominant frequency is not well recovered, whereas Log-SMAE measures the fidelity of the full spectral distribution. The log-SMAE and LFE are formulated as:
\begin{equation}
    \log\mathrm{S}\text{-}\mathrm{MAE} = \frac{1}{\vert\Omega\vert}\sum_{\omega\in\Omega}\lvert\log\frac{P_{\mathrm{GT}}(\omega)}{\sum_{\omega^\prime}P_{\mathrm{GT}}(\omega^\prime)}-\log\frac{P_{\mathrm{pred}}(\omega)}{\sum_{\omega^\prime}P_{\mathrm{pred}}(\omega^\prime)}\rvert
\end{equation}
\begin{equation}
    \mathrm{LFE} = \frac{1}{N}\sum_{i=1}^N\vert f_{\mathrm{GT},i}^{\star}-f_{\mathrm{pred,i}}^{\star}\vert,\qquad f_{i}^{\star} = \frac{1}{2\pi}\mathop{\arg\max}_{\omega\in\Omega}P_i(\omega)
\end{equation}

\begin{table}[tbp]
\centering
\caption{Bandwise S-MAE results of CSDI and HyFAD.}
\label{tab:smae-band}
\begin{sc}
\begin{tabular}{lllllll}
\hline
Missing ratio  & \multicolumn{2}{c}{0.1} & \multicolumn{2}{c}{0.5} & \multicolumn{2}{c}{0.9} \\ \hline
Frequency band & CSDI       & HyFAD      & CSDI       & HyFAD      & CSDI       & HyFAD      \\ \hline
low            & 0.0144     & 0.0143     & 0.0624     & 0.0595     & 0.121      & 0.117      \\
middle         & 0.0131     & 0.0128     & 0.0543     & 0.0522     & 0.0952     & 0.0923     \\
high           & 0.0123     & 0.0120      & 0.0491     & 0.0473     & 0.081      & 0.079      \\ \hline
\end{tabular}
\end{sc}
\end{table}

\begin{table}[tbp]
\caption{Bandwise log-SMAE and LFE results on CSDI and HyFAD, left results: CSDI, right results: HyFAD.}
\label{tab:lsmaelfe}
\begin{sc}
\resizebox{\linewidth}{!}{
\begin{tabular}{ccccccc}
\hline
Missing ratio    & \multicolumn{2}{c}{10\%}      & \multicolumn{2}{c}{50\%}      & \multicolumn{2}{c}{90\%}      \\ \hline
Frequency band   & Log-SMAE      & LFE           & Log-SMAE      & LFE           & Log-SMAE      & LFE           \\ \hline
Low-Frequency    & 0.1046/0.0791 & 4.1e-5/2.3e-5 & 0.3876/0.3113 & 0.0002/0.0002 & 0.6246/0.5569 & 0.0003/0.0003 \\
Middle-Frequency & 0.1271/0.0981 & 0.0012/0.0009 & 0.4531/0.3720 & 0.0048/0.0040 & 0.6827/0.6060 & 0.0074/0.0065 \\
High-Frequency   & 0.1504/0.1184 & 0.0024/0.0019 & 0.4986/0.4084 & 0.0092/0.0074 & 0.7190/0.6701 & 0.0133/0.0120 \\ \hline
\end{tabular}}
\end{sc}
\end{table}
The additional spectral evaluation results provide further empirical evidence for the effect of explicit frequency modeling. As shown in Table.\ref{tab:lsmaelfe}, HyFAD achieves lower Log-SMAE than the pure time-domain counterpart across all three frequency bands and all missing ratios. In the low-frequency band, the relative reductions are 24.4\%, 19.7\%, and 10.8\% under 10\%, 50\%, and 90\% missingness, respectively. In the middle-frequency band, the reductions are 22.8\%, 17.9\%, and 11.2\%, while in the high-frequency band, the reductions are 21.3\%, 18.1\%, and 6.8\%.

HyFAD also consistently reduces LFE in the middle- and high-frequency bands across all missing ratios. Specifically, the relative reductions in the middle-frequency band are 25.0\%, 16.7\%, and 12.2\%, and those in the high-frequency band are 20.8\%, 19.6\%, and 9.8\%. In contrast, the low-frequency LFE values of both methods are already very small, indicating that the dominant low-frequency periodicity can be relatively well recovered even by the pure time-domain baseline. This suggests that the main advantage of HyFAD lies in recovering the more challenging middle- and high-frequency periodic structures.

Overall, these results show that the benefit of explicit frequency modeling is not limited to aggregate improvements in standard time-domain metrics. Instead, HyFAD brings consistent improvements in spectral reconstruction, particularly in the middle- and high-frequency regions where the pure time-domain baseline exhibits larger errors. This further validates the effectiveness of incorporating frequency-domain diffusion for modeling heterogeneous spectral components in time-series imputation.
\subsubsection{Comparison with LSCD and CSDI}
\label{sec:pinterval}
To better demonstrate the superiority of HyFAD over the two next-best baselines, CSDI and LSCD, We present the results of HyFAD, LSCD and CSDI on 5 independent runs and report the corresponding 95\% CI and 
$p$-values in Table.\ref{tab:pinterval}. 

The detailed comparison in Table~\ref{tab:pinterval} further confirms the superiority and stability of HyFAD over the two strongest baselines, LSCD and CSDI. HyFAD achieves the lowest mean RMSE and MAE across all datasets and missingness rates. On PhysioNet, HyFAD consistently outperforms LSCD and CSDI under 10\%, 50\%, and 90\% missingness, with especially clear gains in RMSE. For example, compared with LSCD, HyFAD reduces RMSE by 7.4\%, 5.5\%, and 7.5\% under the three missing ratios, respectively. Similar improvements are observed over CSDI, showing that HyFAD remains robust as the missing ratio increases.

The confidence intervals also indicate that HyFAD maintains stable performance across repeated runs, and its intervals are generally lower than those of LSCD and CSDI. Moreover, all reported p-values are below 0.05, demonstrating that the improvements of HyFAD over both LSCD and CSDI are statistically significant. On the Air Quality dataset, HyFAD also achieves the best RMSE and MAE, although the MAE improvement is relatively smaller. Overall, these results show that HyFAD not only improves average imputation accuracy but also provides statistically reliable gains over competitive diffusion-based baselines.

\begin{table}[tbp]
\centering
\caption{Detailed comparison of HyFAD, LSCD and CSDI.}
\label{tab:pinterval}
\begin{sc}
\resizebox{\linewidth}{!}{
\begin{tabular}{cl|cc|cc|cc|cc}
\hline
\multicolumn{2}{c|}{\multirow{2}{*}{}}           & \multicolumn{2}{c|}{PhysioNet 10\%}     & \multicolumn{2}{c|}{PhysioNet 50\%}     & \multicolumn{2}{c|}{PhysioNet 90\%}     & \multicolumn{2}{c}{Air Quality}           \\ \cline{3-10} 
\multicolumn{2}{c|}{}                            & RMSE               & MAE                & RMSE               & MAE                & RMSE               & MAE                & RMSE                 & MAE                \\ \hline
\multicolumn{1}{c|}{HyFAD} & mean(std)           & 0.423(0.007)       & 0.206(0.001)       & 0.622(0.025)       & 0.288(0.002)       & 0.782(0.028)       & 0.448(0.003)       & 17.730(0.021)        & 9.296(0.007)       \\
\multicolumn{1}{c|}{}      & 95\% CI             & {[}0.414, 0.432{]} & {[}0.206, 0.207{]} & {[}0.591, 0.652{]} & {[}0.286, 0.290{]} & {[}0.746, 0.817{]} & {[}0.444, 0.452{]} & {[}17.705, 17.756{]} & {[}9.286, 9.305{]} \\ \hline
\multicolumn{1}{c|}{LSCD}  & mean(std)           & 0.457(0.004)       & 0.212(0.0004)      & 0.6582(0.0004)     & 0.3046(0.0005)     & 0.845(0.000)       & 0.492(0.000)       & 18.270(0.145)        & 9.340(0.034)       \\
\multicolumn{1}{c|}{}      & 95\% CI             & {[}0.452, 0.462{]} & {[}0.212,0.213{]}  & {[}0.658, 0.659{]} & 0.304, 0.305       & 0.845,0.845        & 0.492,0.492        & 18.125,18.416        & 9.306,9.374        \\
\multicolumn{1}{c|}{}      & p-value(HyFAD-LSCD) & 2.21e-4            & 8.42e-6            & 2.92e-2            & 1.65e-5            & 7.68e-3            & 9.14e-6            & 8.25e-4              & 2.99e-2            \\ \hline
\multicolumn{1}{c|}{CSDI}  & mean(std)           & 0.491(0.004)       & 0.215(0.0005)      & 0.673(0.001)       & 0.307(0.001)       & 0.851(0.000)       & 0.492(0.0005)      & 18.713(0.080)        & 9.347(0.029)       \\
\multicolumn{1}{c|}{}      & 95\% CI             & {[}0.486,0.497{]}  & {[}0.215, 0.216{]} & {[}0.672, 0.675{]} & {[}0.306, 0.307{]} & {[}0.851, 0.851{]} & {[}0.491,0.492{]}  & {[}18.614, 18.812{]} & {[}9.311, 9.382{]} \\
\multicolumn{1}{c|}{}      & p-value(HyFAD-CSDI) & 4.5e-5             & 1.94e-5            & 9.95e-3            & 3.65e-5            & 5.56e-3            & 9.62e-6            & 6.10e-6              & 8.17e-3            \\ \hline
\end{tabular}
}
\end{sc}
\end{table}

\subsubsection{Visualization Results}
\label{sec:appvr}
The imputation results on Air Quality dataset and PhysioNet dataset with missing ratio 10\%, 50\% and 90\% are presented in Fig.\ref{fig:aqi},\ref{fig:phy10},\ref{fig:phy50},\ref{fig:phy90}.  
\begin{figure*}[htbp]
    \centering
    \includegraphics[width=0.9\linewidth]{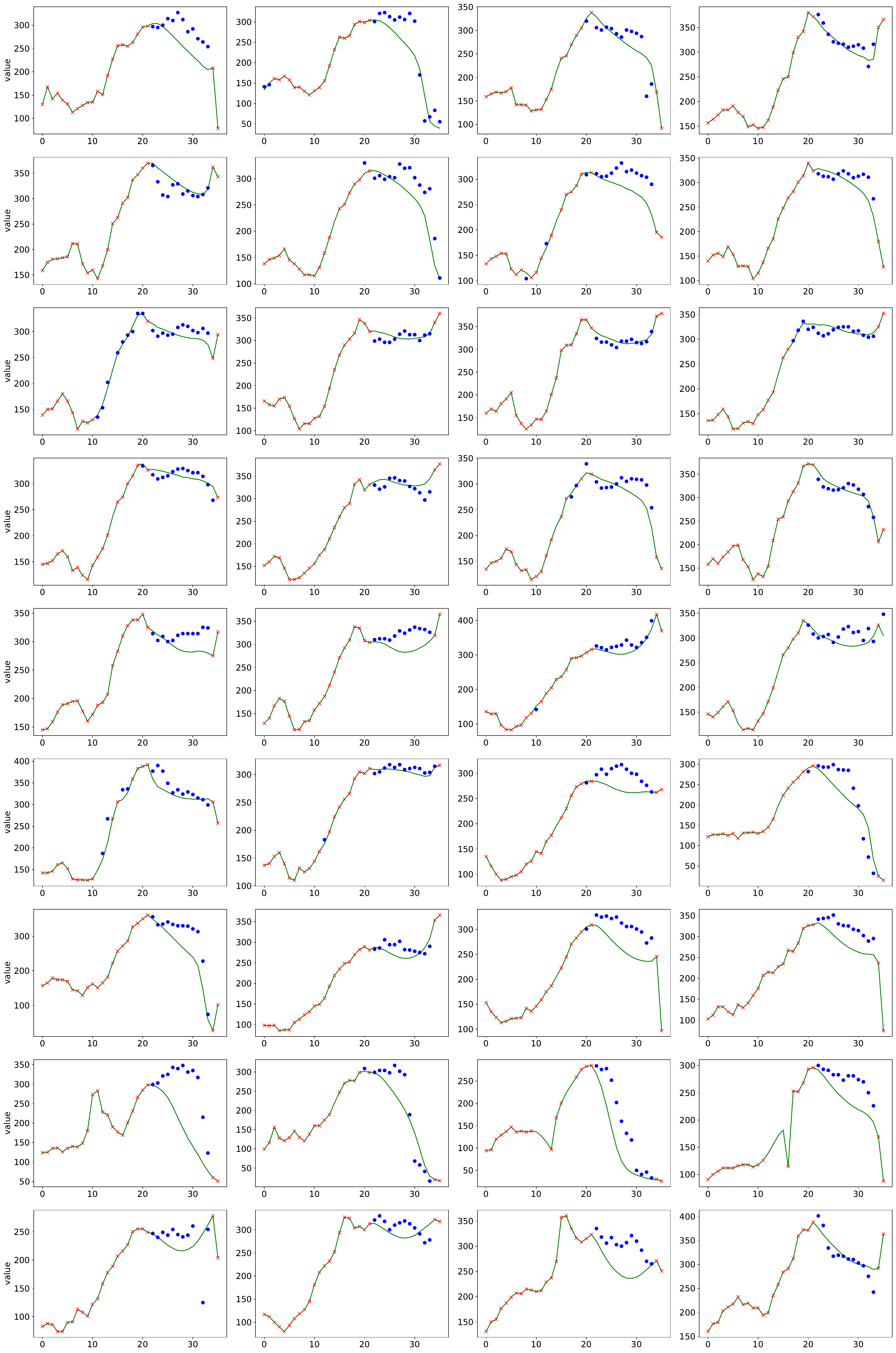}
    \caption{Visualization of imputation results on AQI dataset from Channel 1 to Channel 36. The solid line represents the imputation results, the blue dots represent the ground truth of the missing points, and the red crosses represent the observed values.}
    \label{fig:aqi}
\end{figure*}
\begin{figure*}[htbp]
    \centering
    \includegraphics[width=0.9\linewidth]{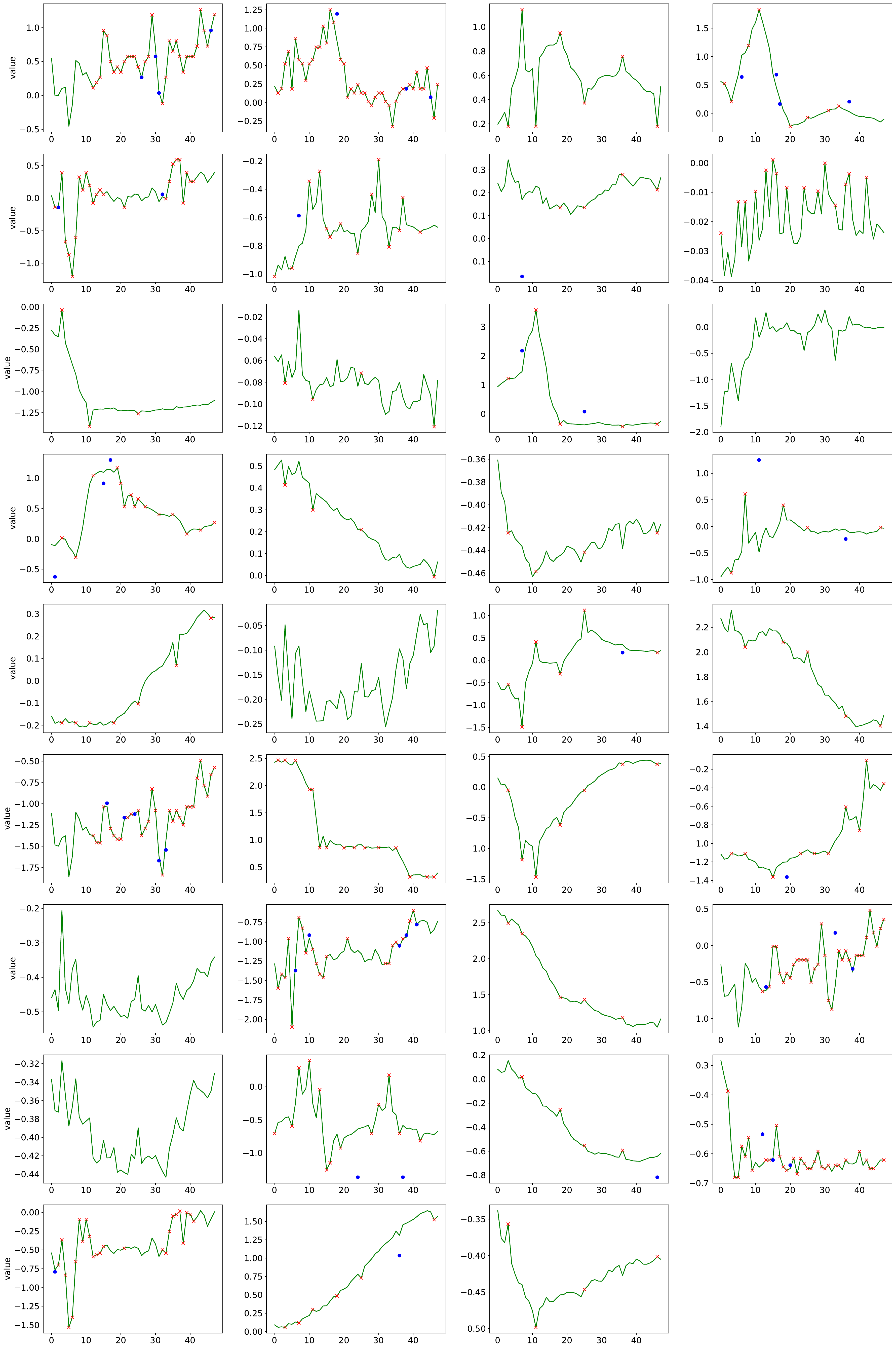}
    \caption{Visualization of imputation results on PhysioNet dataset from Channel 1 to Channel 35 with 10\% missing. The solid line represents the imputation results, the blue dots represent the ground truth of the missing points, and the red crosses represent the observed values.}
    \label{fig:phy10}
\end{figure*}
\begin{figure*}[htbp]
    \centering
    \includegraphics[width=0.9\linewidth]{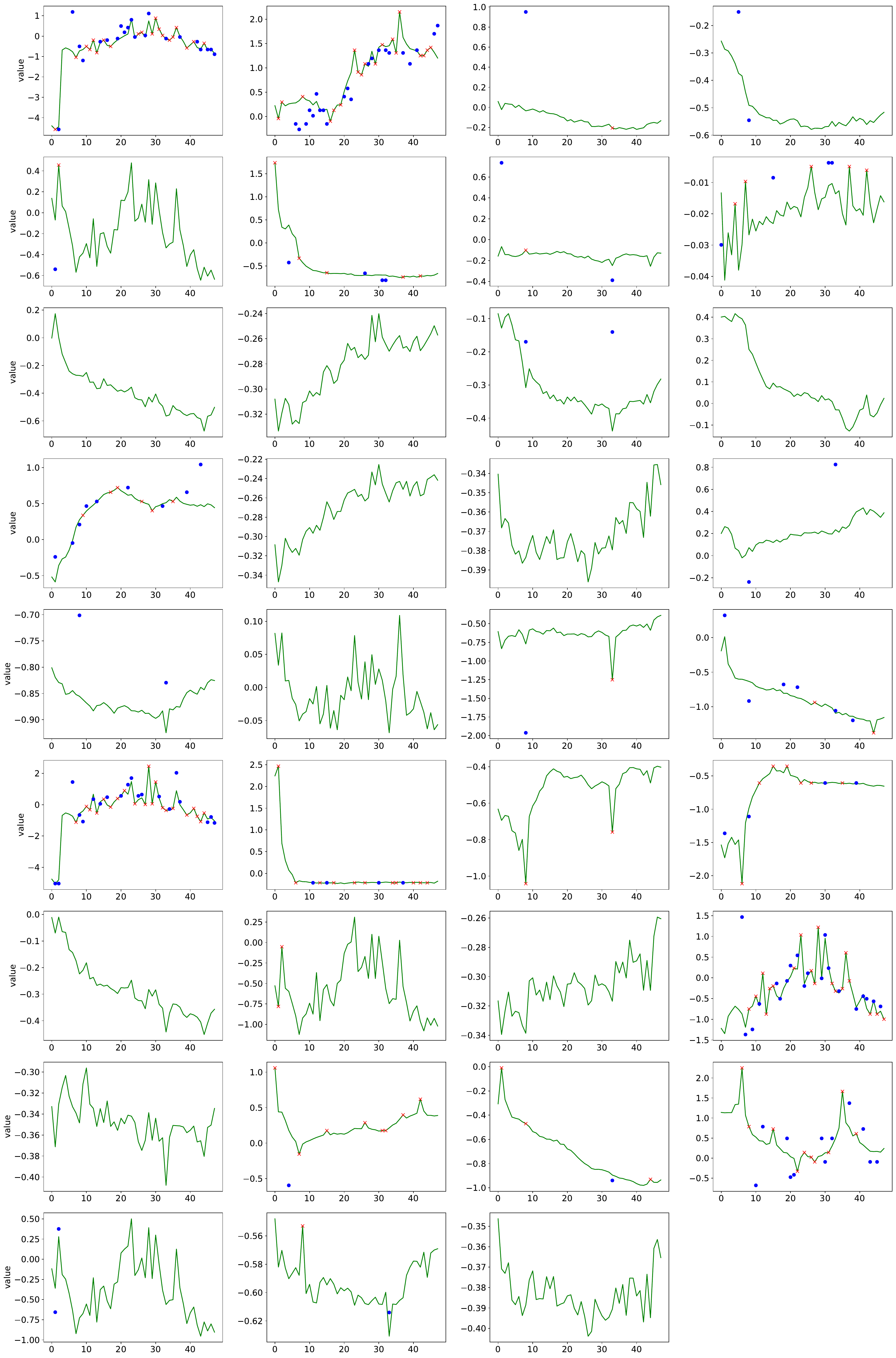}
    \caption{Visualization of imputation results on PhysioNet dataset from Channel 1 to Channel 35 with 50\% missing. The solid line represents the imputation results, the blue dots represent the ground truth of the missing points, and the red crosses represent the observed values.}
    \label{fig:phy50}
\end{figure*}
\begin{figure*}[htbp]
    \centering
    \includegraphics[width=0.9\linewidth]{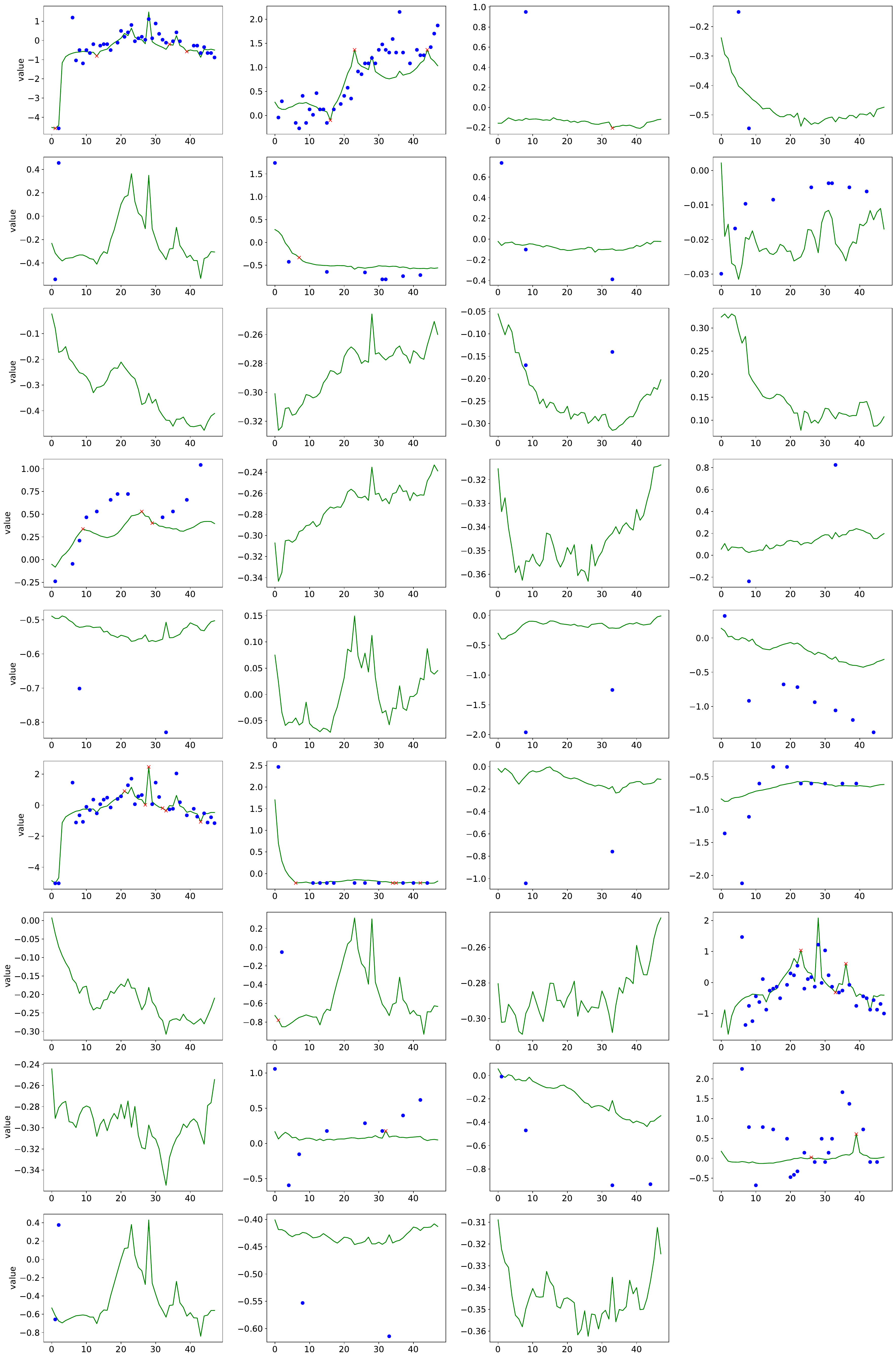}
    \caption{Visualization of imputation results on PhysioNet dataset from Channel 1 to Channel 35 with 90\% missing. The solid line represents the imputation results, the blue dots represent the ground truth of the missing points, and the red crosses represent the observed values.}
    \label{fig:phy90}
\end{figure*}
\subsection{Social impacts}
\label{sec:socimp}
This work may have positive societal impacts by improving time series imputation in domains such as healthcare monitoring, air-quality assessment, industrial sensing, and scientific data analysis, where missing observations are common and better reconstruction can support more reliable downstream analysis. However, the method may also pose risks if imputed values are treated as ground-truth measurements, especially in high-stakes settings such as clinical decision-making, public safety, or critical infrastructure. Incorrect or biased imputations could lead to misleading conclusions or harmful decisions, particularly for underrepresented populations, regions, or sensor conditions. Therefore, imputed values should be clearly distinguished from observed data, and deployment in sensitive domains should include domain-specific validation, uncertainty assessment, privacy protection, and expert review.



\end{document}